\documentclass[11pt]{article}
\pdfoutput=1
\usepackage{enumerate}
\usepackage{pdfsync}
\usepackage[OT1]{fontenc}

\usepackage[usenames]{color}
\usepackage[colorlinks,
            linkcolor=red,
            anchorcolor=blue,
            citecolor=blue
            ]{hyperref}
\usepackage{fullpage}
\usepackage{hyperref}
\usepackage[protrusion=true,expansion=true]{microtype}
\usepackage{pbox}
\usepackage{setspace}
\usepackage{tabularx}
\usepackage{float}
\usepackage{wrapfig,lipsum}
\usepackage{enumitem}
\usepackage{natbib}
\usepackage{multirow}

\makeatletter
\newcommand*{\rom}[1]{\expandafter\@slowromancap\romannumeral #1@}
\makeatother

\usepackage{pifont}
\usepackage{dsfont}
\usepackage{amsthm}
\usepackage{amsmath}
\usepackage{subfigure}
\usepackage{graphicx}
\usepackage{booktabs}
\usepackage{multirow}
\usepackage{threeparttable}

\newtheorem{lemma}{Lemma}
\newtheorem{remark}{Remark}
\newtheorem{theorem}{Theorem}

\newcommand{\nop}[1]{}
\DeclareMathOperator{\ilog}{ilog}

\input{commands}
\def \dd{\text{d}}

\allowdisplaybreaks
\begin{document}
\title{\huge  MOTS: Minimax Optimal Thompson Sampling}
\author
{
	Tianyuan Jin\thanks{School of Computing, National University of Singapore, Singapore; e-mail: {\tt Tianyuan1044@gmail.com}} 
	,~
	Pan Xu\thanks{Department of Computer Science, University of California, Los Angeles, Los Angeles, CA 90095; e-mail: {\tt panxu@cs.ucla.edu}} 
	,~
	Jieming Shi\thanks{School of Computing, National University of Singapore, Singapore; e-mail: {\tt Shijm@nus.edu.sg}} 
    ,~
	Xiaokui Xiao\thanks{School of Computing, National University of Singapore, Singapore; e-mail: {\tt xkxiao@nus.edu.sg}} 
	,~
	Quanquan Gu\thanks{Department of Computer Science, University of California, Los Angeles, Los Angeles, CA 90095; e-mail: {\tt qgu@cs.ucla.edu}}
}
\date{}
\maketitle

\begin{abstract}
Thompson sampling is one of the most widely used algorithms for many online decision problems, due to its simplicity in implementation and superior empirical performance over other state-of-the-art methods. Despite its popularity and empirical success, it has remained an open problem whether Thompson sampling can match the minimax lower bound $\Omega(\sqrt{KT})$ for $K$-armed bandit problems, where $T$ is the total time horizon. In this paper, we solve this long open problem by proposing a variant of Thompson sampling called MOTS that adaptively clips the sampling instance of the chosen arm at each time step. We prove that this simple variant of Thompson sampling achieves the minimax optimal regret bound $O(\sqrt{KT})$ for finite time horizon $T$, as well as the asymptotic optimal regret bound for Gaussian rewards when $T$ approaches infinity. To our knowledge, MOTS is the first Thompson sampling type algorithm that achieves the minimax optimality for multi-armed bandit problems.
\end{abstract}

\section{Introduction}
The Multi-Armed Bandit (MAB) problem is a sequential decision process which is typically described as a game between the agent and the environment with $K$ arms. The game proceeds in $T$ time steps. In each time step $t=1,\ldots,T$, the agent plays an arm $A_t\in \{1,2,\cdots,K\}$ based on the observation of the previous $t-1$ time steps, and then observes a reward $r_t$ that is independently generated from a 1-subGaussian distribution with mean value $\mu_{A_t}$,  where $\mu_1,\mu_2,\cdots,\mu_{K}\in \mathbb{R}$ are unknown.
The goal of the agent is to maximize the  cumulative reward over $T$ time steps.
The performance of a strategy for MAB is measured  by the expected cumulative difference over $T$ time steps between playing the best arm and playing the arm according to the strategy, which is also called the regret of a bandit strategy. Formally, the regret $R_{\mu}(T)$ is defined as follows
\begin{equation}
    {R_{\mu}(T)= T \cdot \max_{i\in \{1,2,\cdots, K \}}\mu_i -\mathbb{E}_{\mu} \Bigg[ \sum_{t=1}^T r_t\Bigg].}
\end{equation}
For a fixed time horizon $T$, the problem-independent lower bound~\citep{auer2002nonstochastic} states that any strategy has at least a regret in the order of $\Omega(\sqrt{KT})$, which is called the \emph{minimax optimal} regret. On the other hand, for a fixed model (i.e., $\mu_1,\ldots,\mu_K$ are fixed), \citet{lai1985asymptotically} proved 
that any strategy must have at least  $C(\mu)\log(T)(1-o(1))$ regret when the horizon $T$ approaches infinity, where $C(\mu)$ is a constant depending on the model. Therefore, a strategy with a regret upper-bounded by $C(\mu)\log(T)(1-o(1))$ is \emph{asymptotically optimal}.

This paper studies the earliest bandit strategy, Thompson sampling (TS)~\citep{thompson1933likelihood}. It has been observed in practice that TS often achieves a smaller regret than many upper confidence bound (UCB)-based algorithms \citep{chapelle2011empirical,wang2018thompson}. In addition, TS
is simple and easy to implement. Despite these advantages, the theoretical analysis of TS algorithms has not been established until the past decade. In particular, in the seminal work by \citet{agrawal2012analysis}, they  provided the  first finite-time analysis of TS.  \citet{kaufmann2012thompson} and \citet{agrawal2013further}  showed that the regret bound of TS is asymptotically optimal when using Beta priors. Subsequently, \citet{agrawal2017near} showed that TS with Beta priors achieves an $O(\sqrt{KT\log T})$ problem-independent regret bound while maintaining the asymptotic optimality. In addition, they proved that TS with Gaussian priors can achieve an improved regret bound $O(\sqrt{KT\log K})$. 
\citet{agrawal2017near} also established the following regret lower bound for TS: the TS strategy with Gaussian priors has a worst-case regret  $\Omega(\sqrt{KT\log K})$.

\noindent\textbf{Main Contributions.} It remains an open problem~\citep{li2012open} whether TS type algorithms can achieve the minimax optimal regret bound $O(\sqrt{KT})$ for MAB problems. In this paper, we solve this open problem by proposing a variant of Thompson sampling, referred to as Minimax Optimal Thompson Sampling (MOTS), which clips the sampling instances for each arm based on the history of pulls. 
We prove that MOTS achieves $O(\sqrt{KT})$ problem-independent regret, which is minimax optimal and improves the existing best result, i.e., $O(\sqrt{KT\log K})$.
Furthermore, we show that when the reward distributions are Gaussian, a variant of MOTS with clipped Rayleigh distributions, namely MOTS-$\cJ$, can simultaneously achieve asymptotic and minimax optimal regret bounds. 
Our result also conveys the important message that the lower bound for TS with Gaussian priors in \cite{agrawal2017near} may not always hold in the more general cases when non-Gaussian priors are used. Our experiments demonstrate the superiority of MOTS over the state-of-the-art bandit algorithms such as UCB~\citep{auer2002finite}, MOSS~\citep{audibert2009minimax}, and  TS~\citep{thompson1933likelihood} with Gaussian priors.

We provide a detailed comparison on the minimax optimality and asymptotic optimality of Thompson sampling type algorithms in Table \ref{tab:regret_comparison}.

\noindent\textbf{Notations.} 
A random variable $X$ is said to follow a 1-subGaussian distribution, if it holds that $\EE_{X}[\exp(\lambda X-\lambda \EE_{X}[X])]\leq \exp(\lambda^2/2)$ for all $\lambda \in \RR$. We denote $\log^{+}(x)=\max\{0,\log x \}$. We let $T$ be the total number of time steps, $K$ be the number of arms, and $[K] = \{1,2,\cdots,K\}$. Without loss of generality, we assume that $\mu_1=\max_{i\in[K]} \mu_i$ throughout this paper. We use $\Delta_i$ to denote the gap between arm $1$ and arm $i$, i.e., $\Delta_i:=\mu_1-\mu_i$, $i\in[K]\setminus\{1\}$. 
We denote $T_{i}(t):=\sum_{j=1}^{t}\ind\{A_j=i\}$ as the number of times that arm $i$ has been played at time step $t$, and $\hat{\mu}_{i}(t):=\sum_{j=1}^t\ind{\{A_j=i\}}\cdot r_j/T_i(t)$ as the average reward for pulling arm $i$ up to time $t$, where $r_j$ is the reward received by the algorithm at time $j$.   


\begin{table}[ht]
    \centering
    \caption{Comparisons of different TS type algorithms. The \emph{minimax ratio} is the ratio (up to constant factors) between the problem-independent regret bound of the algorithm and the minimax optimal regret $O(\sqrt{KT})$. For instance, when the ration equals $1$, it is minimax optimal; otherwise, it is minimax suboptimal. 
    The results in  \cite{kaufmann2012thompson,agrawal2013further,agrawal2017near} are obtained for rewards bounded in $[0,1]$, but the techniques in their papers also work for Gaussian rewards (See  \cite{korda2013thompson} for details).  \label{tab:regret_comparison}}    
\begin{threeparttable} 
    \begin{tabular}{ccccl}
    \toprule
        & Reward Type& Minimax Ratio & Asym. Optimal & Reference\\
    \midrule
    \multirow{3}{*}{TS}
    & Bernoulli &--&Yes & \cite{kaufmann2012thompson} \\ 
     & Bernoulli & $\sqrt{\log T}$ & Yes &  \cite{agrawal2013further}\\
    & Bernoulli &$\sqrt{\log K}$ \tnote{*} & -- &\cite{agrawal2017near}\\
    \cmidrule{2-4}
    \multirow{2}{*}{MOTS} & subGaussian & $1$ &  No\tnote{**}& $\triangleright$ Theorems \ref{theorem:minimax-optimal}, \ref{theorem:asymptotic-optimal}\\
    & subGaussian &$\ilog^{(m-1)}(T)$ \tnote{***}& Yes & $\triangleright$ Theorem \ref{theorem-new3}\\
    \cmidrule{2-4}
    MOTS-$\cJ$ & Gaussian &$1$ & Yes &{$\triangleright$ Theorem \ref{theorem_J_distribution}}\\
    \bottomrule
    \end{tabular}%
    \begin{tablenotes}\footnotesize
    \item[*] It has been proved by \citet{agrawal2017near} that the $\sqrt{\log K}$ term in the problem-independent regret is unimprovable for Thompson sampling using Gaussian priors.
    \item[**] 
    As is shown in Theorem \ref{theorem:asymptotic-optimal}, MOTS is asymptotically optimal up to a multiplicative factor $1/\rho$, where $\rho\in(1/2,1)$ is a fixed constant.
    \item[***] $\ilog^{(r)}(\cdot)$ is the iterated logarithm of order $r$, and $m\geq 2$ is an arbitrary integer independent of $T$.
    \end{tablenotes}
\end{threeparttable}
\end{table}

\section{Related Work}
Existing works on regret minimization for stochastic bandit problems mainly consider two notions of optimality: asymptotic optimality and minimax optimality. UCB~\citep{garivier2011kl,maillard2011finite}, Bayes UCB~\citep{kaufmann2016bayesian}, and Thompson sampling~\citep{kaufmann2012thompson,agrawal2017near,korda2013thompson} are all shown to be asymptotically optimal. Meanwhile, MOSS~\citep{audibert2009minimax} is the first method proved to be minimax optimal. Subsequently, two UCB-based methods,  AdaUCB~\citep{lattimore2018refining} and KL-UCB$^{++}$~\citep{menard2017minimax}, are also shown to achieve minimax optimality.
In addition, AdaUCB is proved to be almost instance-dependent optimal for Gaussian multi-armed bandit problems \citep{lattimore2018refining}.

There are many other methods on regret minimization for stochastic bandits, including explore-then-commit~\citep{auer2010ucb,perchet2016batched},  $\epsilon$-Greedy~\citep{auer2002finite}, and RandUCB~\citep{vaswani2019old}. However, these methods are proved to be suboptimal~\citep{auer2002finite,garivier2016explore,vaswani2019old}. One exception is the recent proposed double explore-then-commit algorithm \citep{jin2020double}, which achieves asymptotic optimality. 
Another line of works study different variants of the problem setting, such as the batched bandit problem~\citep{NIPS2019_8341}, and bandit with delayed feedback~\citep{pike2018bandits}. We refer interested readers to \cite{lattimore2018bandit} for a more comprehensive overview of bandit algorithms.

For Thompson sampling, \cite{russo2014learning} studied the Bayesian regret and \citet{bubeck2013prior} improved it to $O(\sqrt{KT})$ using the confidence bound analysis of MOSS \citep{audibert2009minimax}. However, it should be noted that the Bayesian regret is known to be less informative than the frequentist regret $R_{\mu}(T)$ studied in this paper. In fact, one can show that our minimax optimal regret $R_{\mu}(T)=O(\sqrt{KT})$ immediately implies that the Bayesian regret is also in the order of $O(\sqrt{KT})$, but the reverse is not true \citep{lattimore2018bandit}. We refer interested readers to \citet{russo2018tutorial} for a thorough introduction of Thompson sampling and its various applications.

\section{Minimax Optimal Thompson Sampling Algorithm} \label{sec:algo}

\subsection{General Thompson sampling strategy}\label{sec:general_TS}
We first describe the general Thompson sampling (TS) strategy. In the first $K$ time steps, TS plays each arm $i\in[K]$ once, and updates the average reward estimation $\hat\mu_i(K+1)$ for each arm $i$. (This is a standard warm-start in the bandit literature.) Subsequently, the algorithm maintains a distribution $D_i(t)$ for each arm $i\in[K]$ at time step $t=K+1,\ldots,T$, whose update rule will be elaborated shortly. At step $t$, the algorithm samples instances $\theta_i(t)$ independently from distribution $D_i(t)$, for all $i\in[K]$. Then, the algorithm plays the arm that maximizes $\theta_i(t)$: $A_t=\argmax_{i\in[K]}\theta_i(t)$, and receives a reward $r_t$. The average reward $\hat\mu_i(t)$ and the number of pulls $T_i(t)$ for arm $i\in[K]$ are then updated accordingly.

We refer to algorithms that follow the general TS strategy described above (e.g., those in \cite{chapelle2011empirical,agrawal2017near}) as {\it TS type algorithms}. Following the above definition, our MOTS method is a TS type algorithm, but it differs from other algorithms of this type in the choice of distribution $D_i(t)$: 
existing algorithms (e.g., \cite{agrawal2017near}) typically use Gaussian or Beta distributions as the posterior distribution, whereas MOTS uses a {\it clipped} Gaussian distribution, which we detail in Section~\ref{sec:sample_theta}. Nevertheless, we should note that MOTS fits exactly into the description of Thompson sampling in \citet{li2012open,chapelle2011empirical}.

\subsection{Thompson sampling using clipped Gaussian distributions}
\label{sec:sample_theta}

\begin{algorithm}[t]
   \caption{Minimax Optimal Thompson Sampling with Clipping (MOTS)}
   \label{alg:mots}
\begin{algorithmic}[1]
   \STATE\textbf{Input:} Arm set $[K]$.
   \STATE\textbf{Initialization:} Play arm once and set $T_i(K+1)=1$; let $\hat{\mu}_{i}(K+1)$ be the observed  reward of playing arm $i$
   \FOR{$t=K+1,K+2,\cdots, T$}
   \STATE For all $i\in[K]$, sample $\theta_i(t)$ independently from   $D_{i}(t)$,   which is defined in Section \ref{sec:sample_theta}\label{algline:distribution_Dit} 
   \STATE Play arm $A_t= \arg \max_{i\in[K]} \theta_i(t)$ and observe the reward $r_t$
   \STATE For all $i\in[K]$ $$\hat\mu_{i}(t+1)=\frac{T_{i}(t) \cdot \hat{\mu}_{i}(t)+r_t\ind\{i=A_t\}}{T_{i}(t)+\ind\{i=A_t\}}$$ 
   \STATE For all $i\in[K]$:   $T_i(t+1)=T_i(t)+\ind\{i=A_t\}$ 
    \ENDFOR
\end{algorithmic}
\end{algorithm}

Algorithm~\ref{alg:mots} shows the pseudo-code of MOTS, with $D_i(t)$ formulated as follows. First, at time step $t$, for all arm $i\in[K]$, we define a {\it confidence range} $(-\infty, \tau_i(t))$, where
\begin{equation} 
\small
\label{eq:shrink_range}
\tau_i(t)= \hat{\mu}_{i}(t)+\sqrt{\frac{\alpha}{T_{i}(t)}\log^{+} \bigg(\frac{T}{KT_{i}(t)}\bigg)},
\end{equation}
$\log^{+}(x)=\max\{0,\log x\}$, and $\alpha>0$ is a constant. Given $\tau_i(t)$, we first sample an instance $\tilde\theta_i(t)$ from Gaussian distribution $\mathcal{N}(\hat{\mu}_{i}(t),1/(\rho T_{i}(t)))$, where $\rho\in (1/2,1)$ is a tuning parameter (The intuition  could be found at Lemma~\ref{lem:decom=false}). Then, we set $\theta_i(t)$ in Line~\ref{algline:distribution_Dit} of Algorithm \ref{alg:mots} as follows:
\begin{equation}\label{eq:choose_of_theta}
\theta_i(t)= \min\big\{\tilde\theta_i(t), \; \tau_i(t)\big\}.
\end{equation}


In other words, $\theta_i(t)$ follows a {\it clipped} Gaussian distribution with the following PDF: 
\begin{align}
  f(x)=  \begin{cases}
    \varphi\left( x \mid \hat{\mu}_{i}(t), \frac{1}{\rho T_i(t)}\right) + \left(1 - \Phi\left(x \mid \hat{\mu}_{i}(t), \frac{1}{\rho T_i(t)}\right)\right) \cdot \delta\left(x - \tau_i(t) \right), & \text{if }x \le \tau_i(t);\\
     0, &\text{otherwise,}
    \end{cases}
\end{align}
where  $\varphi( x \mid \mu, \sigma^2)$ and $\Phi( x \mid \mu, \sigma^2 )$ denote the PDF and CDF of $\mathcal{N}(\mu, \sigma^2)$, respectively, and $\delta(\cdot)$ is the Dirac delta function.

MOTS uses $\theta_i(t)$ as the estimate for arm $i$ at time step $t$, and plays the arm with the largest estimate. That is, MOTS utilizes $\tilde\theta_i(t)$ directly as an estimate if it is not larger than $\tau_i(t)$ (i.e., if it does not deviate too much from the observed average reward $\hat\mu_i(t)$);
otherwise, MOTS clips $\tilde\theta_i(t)$ and reduces it to $\tau_i(t)$. The rationale of this clipping is that if $\tilde\theta_i(t)$ deviates considerably from $\hat\mu_i(t)$, then it is likely to be an overestimation of arm $i$'s actual reward; in that case, it is sensible to use a reduced version of $\tilde\theta_i(t)$ as an improved estimate for arm $i$. The challenge, however, is that we need to carefully decide $\tau_i(t)$, so as to ensure the asymptotic and minimax optimality. In Section~\ref{sec:theory}, we will show that our choice of $\tau_i(t)$ addresses this challenge.

\section{Theoretical Analysis of MOTS} \label{sec:theory}
\subsection{Regret of MOTS for subGaussian rewards}
We first show that MOTS is minimax optimal.
\begin{theorem}[Minimax Optimality of MOTS]\label{theorem:minimax-optimal}
Assume that the reward of  arm $i \in [K]$ is 1-subGaussian with mean $\mu_i$. For any fixed $\rho\in (1/2,1)$ and $\alpha\geq 4$,  the regret of Algorithm~\ref{alg:mots} satisfies 
\begin{align}\label{eq:mini1}
    R_{\mu}(T)= O\bigg(\sqrt{KT}+\sum_{i=2}^K \Delta_i\bigg).
\end{align}
\end{theorem}
The second term on the right hand side of~\eqref{eq:mini1} is due to the fact that we need to pull each arm at least once in Algorithm \ref{alg:mots}. Following the convention in the literature~\citep{audibert2009minimax,agrawal2017near}, we only need to consider the case when $\sum_{i=2}^{K}\Delta_i$ is dominated by $\sqrt{KT}$. 

\vspace{1mm}
\begin{remark}
Compared with the results in \cite{agrawal2017near}, the regret bound of MOTS improves that of TS with Beta priors by a factor of $O(\sqrt{\log T})$, and that of TS with Gaussian priors by a factor of $O(\sqrt{\log K})$. To the best of our knowledge, MOTS is the first TS type algorithm that achieves the minimax optimal regret $\Omega(\sqrt{KT})$ for MAB problems \citep{auer2002finite}. 
\end{remark}
The next theorem presents the asymptotic regret bound of MOTS for subGaussian rewards. 
\begin{theorem}\label{theorem:asymptotic-optimal}
Under the same conditions in Theorem \ref{theorem:minimax-optimal}, the regret $R_{\mu}(T)$ of Algorithm~\ref{alg:mots} satisfies 
\begin{equation}\label{eq:asymptotic_rate_mots}
    \lim_{T \rightarrow\infty}\frac{R_{\mu}(T)}{\log (T)}=\sum_{i:\Delta_i>0}\frac{2}{\rho \Delta_i}.
\end{equation}
\end{theorem}
\citet{lai1985asymptotically} proved that for Gaussian rewards, the asymptotic regret rate $\lim_{T\rightarrow\infty}R_{\mu}/\log T$ is at least $\sum_{i:\Delta_i>0} 2/\Delta_i$. Therefore, Theorem \ref{theorem:asymptotic-optimal} indicates that the asymptotic regret rate of MOTS matches the aforementioned lower bound up to a multiplicative factor $1/\rho$, where $\rho\in(1/2,1)$ is arbitrarily fixed. 

In the following theorem, by setting $\rho$ to be time-varying, we show that MOTS is able to exactly match the asymptotic lower bound. 
\begin{theorem}
\label{theorem-new3}
Assume the reward of each arm $i$ is 1-subGaussian with mean $\mu_i$, $i\in[K]$.  In Algorithm \ref{alg:mots}, if we choose $\alpha\geq 4$ and $\rho=1-(\ilog^{(m)}(T)/40)^{-1/2}$, then the regret of MOTS satisfies
\begin{align}
    R_{\mu}(T)= O\bigg(\sqrt{KT}\ilog^{(m-1)}(T)+\sum_{i=2}^{K}\Delta_i\bigg),\quad\text{and }\lim_{T \rightarrow\infty}\frac{R_{\mu}(T)}{\log (T)}=\sum_{i:\Delta_i>0}\frac{2}{ \Delta_i},
\end{align}
where $m\geq 2$ is an arbitrary integer  independent of $T$ and
$\ilog^{(m)}(T)$ is the result of iteratively applying the logarithm function on $T$ for $m$ times, i.e., $\ilog^{(m)}(x)=\max\big\{\log \big(\ilog^{(m-1)}(x) \big), e \big\}$ and $\ilog^{(0)}(a)=a$.
\end{theorem}
Theorem~\ref{theorem-new3} indicates that MOTS can exactly match the asymptotic lower bound in \cite{lai1985asymptotically}, at the cost of forgoing minimax optimality by up to a factor of $O(\ilog^{(m-1)}(T))$. For instance, if we choose $m=4$, then MOTS is minimax optimal up to a factor of $O(\log \log \log T)$. Although this problem-independent bound is slightly worse than that in Theorem~\ref{theorem:minimax-optimal}, it is still a significant improvement over the best known problem-independent bound $O(\sqrt{KT\log T})$ for asymptotically optimal TS type algorithms \citep{agrawal2017near}. 

Finally, it should be noted that the lower bound of the asymptotic regret rate $\lim_{T\rightarrow\infty}R_{\mu}/\log T\geq\sum_{i:\Delta_i>0} 2/\Delta_i$ in \citet{lai1985asymptotically} was established for Gaussian rewards. Since Gaussian is a special case of subGaussian, the lower bound for the Gaussian case is also a valid lower bound for general subGaussian cases. Therefore, MOTS is asymptotically optimal. Similar arguments are widely adopted in the literature \citep{lattimore2018bandit}.



\subsection{Regret of MOTS for Gaussian rewards}
In this subsection, we present a variant of MOTS, called MOTS-$\cJ$, which simultaneously achieves the minimax and asymptotic optimality when the reward distribution is Gaussian. 

\begin{algorithm}[t]
   \caption{MOTS-$\cJ$}
   \label{alg:mots_J}
\begin{algorithmic}[1]
   \STATE\textbf{Input:} Arm set $[K]$.
   \STATE\textbf{Initialization:} Play arm once and set $T_i(K+1)=1$; let $\hat{\mu}_{i}(K+1)$ be the observed  reward of playing arm $i$
   \FOR{$t=K+1,K+2,\cdots, T$}
   \STATE For all $i\in[K]$, sample $\theta_i(t)$ independently from   $D_{i}(t)$ as follows: sample $\tilde\theta_{i}(t)$ from $\cJ(\hat\mu_i(t),1/T_i(t))$; set $\theta_i(t)=\min\{\tilde\theta_i(t),\tau_i(t)\}$, where $\tau_i(t)$ is defined in \eqref{eq:shrink_range}
   \label{algline_j:distribution_Dit} 
   \STATE Play arm $A_t= \arg \max_{i\in[K]} \theta_i(t)$ and observe the reward $r_t$
   \STATE For all $i\in[K]$ $$\hat\mu_{i}(t+1)=\frac{T_{i}(t) \cdot \hat{\mu}_{i}(t)+r_t\ind\{i=A_t\}}{T_{i}(t)+\ind\{i=A_t\}}$$ 
   \STATE For all $i\in[K]$:   $T_i(t+1)=T_i(t)+\ind\{i=A_t\}$   
    \ENDFOR
\end{algorithmic}
\end{algorithm}

Algorithm~\ref{alg:mots_J} shows the pseudo-code of MOTS-$\cJ$. Observe that MOTS-$\cJ$ is identical to MOTS, except that in Line~\ref{algline_j:distribution_Dit} of MOTS-$\cJ$, it samples $\tilde\theta_i(t)$ from a distribution $\cJ(\hat\mu_i(t), 1/T_i(t))$ instead of the Gaussian distribution used in Section \ref{sec:sample_theta} for MOTS. The distribution $\cJ(\mu,\sigma^2)$ has the following PDF:
\begin{align}\label{eq:def_pdf_of_J}
    \phi_{\cJ}(x)=\frac{1}{2\sigma^2}\cdot |x-\mu|\cdot \exp\bigg[{-\frac{1}{2}\bigg(\frac{x-\mu}{\sigma}\bigg)^2}\bigg].
\end{align}
Note that 
$\cJ$ is a Rayleigh distribution if it is restricted to $x\geq0$.

The following theorem shows the minimax and asymptotic optimality of MOTS-$\cJ$ for Gaussian rewards.
\begin{theorem}
\label{theorem_J_distribution}
Assume that the reward of each arm $i$ follows a Gaussian distribution $\cN(\mu_i,1)$, and that $\alpha \ge 2$ in \eqref{eq:shrink_range}. 
The regret of MOTS-$\cJ$ satisfies
\begin{align}
     R_{\mu}(T)= O\bigg(\sqrt{KT}+\sum_{i=2}^K \Delta_i\bigg), \quad\text{and }
    \lim_{T \rightarrow\infty}\frac{R_{\mu}(T)}{\log (T)}=\sum_{i:\Delta_i>0}\frac{2}{ \Delta_i}.
\end{align}
\end{theorem}

\begin{remark}
To our knowledge, MOTS-$\cJ$ is the first TS type algorithm that simultaneously achieves the minimax and asymptotic optimality. Though the clipping threshold of  MOTS-$\cJ$ in \eqref{eq:shrink_range} looks like the MOSS index in \citet{audibert2009minimax}, there are some key differences in the choice of $\alpha$, the theoretical analysis and the result. Specifically, \citet{audibert2009minimax} proved that MOSS with the exploration index $\alpha=4$ achieves minimax optimality for MAB. It remained an open problem how to improve MOSS to be both minimax and asymptotically optimal until \citet{menard2017minimax} proposed the KL-UCB$^{++}$ algorithm for exponential families of distributions which implies that MOSS with exploration index $\alpha=2$ could lead to the asymptotic optimal regret for Gaussian rewards. For more details on the choice of $\alpha$ in MOSS, we refer interested readers to the discussion in Chapter 9.3 of \citet{lattimore2018bandit}. 

Compared with MOSS index based UCB algorithms, our proposed MOTS-$\cJ$ is both minimax and asymptotically optimal as long as $\alpha\geq 2$. This flexibility is due to the fact that our theoretical analysis (asymptotic optimal part) based on Thompson sampling is quite different from those based on UCB in \citet{audibert2009minimax,menard2017minimax}. Not confined by the choice of the exploration index  $\alpha$, it will be more suitable to design better algorithms based on MOTS-$\cJ$, e.g., achieving instance-dependent optimality (see \cite{lattimore2015optimally} for details) while keeping the asymptotic optimality.
\end{remark}

\section{Proof of the Minimax Optimality of MOTS}\label{sec:proof_minimaxi}
In what follows, we prove our main result in Theorem~\ref{theorem:minimax-optimal}, and we defer the proofs of all other results to the appendix. We first present several useful lemmas. 
Lemmas \ref{lemma:high_prob_all_positive} and  \ref{lem:sumTpro} characterise the concentration properties of subGaussian random variables.
\begin{lemma}[Lemma 9.3 in~\cite{lattimore2018bandit}]\label{lemma:high_prob_all_positive}
Let  $X_1, X_2,\cdots$ be independent and  $1$-subGaussian random variables with zero means. Denote $\hat\mu_t=1/t\sum_{s=1}^{t}X_s$. Then, for $\alpha\geq 4$ and any $\Delta>0$,
\begin{equation}
    \begin{split}
    \mathbb{P}\Bigg(\exists\ s\in [T]: \hat{\mu}_s+& \sqrt{\frac{\alpha}{s}\log^+\bigg( \frac{T}{sK }\bigg)}+\Delta\leq 0\Bigg)\leq \frac{15 K}{T\Delta^2}. \\
    \end{split}
    \end{equation}

\end{lemma}
\begin{lemma}\label{lem:sumTpro}
Let $\omega>0$ be a constant and $X_1, X_2,\ldots,X_n$ be independent and 1-subGaussian random variables with zero means. Denote $\hat{\mu}_n=1/n\sum_{s=1}^n X_s$. Then, for $\alpha>0$ and any $N \leq T$, 
\begin{align}
       \sum_{n=1}^T \mathbb{P}\bigg(\hat{\mu}_{n}+\sqrt{\frac{\alpha}{n}\log^{+}\bigg(\frac{N}{n}\bigg)} \geq \omega \bigg)   
       &\leq  1+\frac{\alpha\log^{+}({N}{\omega^2})}{\omega^2} +\frac{3}{\omega^2}+\frac{\sqrt{2\alpha\pi {\log^{+}({N}{\omega^2})}}}{\omega^2}.
\end{align}
\end{lemma}

Next, we introduce a few notations for ease of exposition. 
Recall that we have defined $\hat\mu_i(t)$ to be the average reward for arm $i$ up to a time $t$. Now, let $\hat\mu_{is}$ be the average reward for arm $i$ up to when it is played the $s$-th time. In addition, similar to the definitions of $D_i(t)$ and $\theta_i(t)$, we define $D_{is}$ as the distribution of arm $i$ when it is played the $s$-th time, and $\theta_{is}$ as a sample from distribution $D_{is}$.



The following lemma upper bounds the expected total number of pulls of each arm at time $T$. We note that this lemma is first proved by~\citet{agrawal2017near}; here, we use an improved version presented in~\cite{lattimore2018bandit}\footnote{Since MOTS plays every arm once at the beginning, \eqref{eq:minimax-arm-i} starts with $t=K+1$ and $s=1$.}. 
\begin{lemma}[Theorem 36.2 in  \cite{lattimore2018bandit}]
\label{lem:thom-bound}
Let $\epsilon \in \mathbb{R}^{+}$. Then, the expected number of times that Algorithm~\ref{alg:mots} plays arm $i$ is bounded by 
{\begin{align} 
     \mathbb{E}[T_i(T)] & =   \EE\bigg[ \sum_{t=1}^T \ind \{A_t=i, E_{i}(t) \} \bigg]   + \EE\bigg[ \sum_{t=1}^T \ind \{A_t=i, E_{i}^c (t) \} \bigg] \notag   \\
     & \leq 1+\mathbb{E}\bigg[\sum_{s=1}^{T-1}\left(\frac{1}{G_{1s}(\epsilon)}-1\right) \bigg]  +\mathbb{E}\bigg[\sum_{t=K+1}^{T-1}\ind\{ A_t=i,E_{i}^{c}(t)\}\bigg] \label{eq:minimax-arm-i} \\
 &  \leq 2+\mathbb{E}\bigg[\sum_{s=1}^{T-1}\left(\frac{1}{G_{1s}(\epsilon)}-1\right) \bigg]   + \EE \bigg[ \sum_{s=1}^{T-1} \ind\{ G_{is}(\epsilon)>1/T\} \label{eq:asymptotically-arm-i} \bigg],
\end{align}}%
where $G_{is}(\epsilon)=1-F_{is}(\mu_1-\epsilon)$, $F_{is}$ is the CDF of $D_{is}$, and $E_{i}(t)=\{\theta_i(t)\leq \mu_1-\epsilon \}$.
\end{lemma}
Based on the decomposition of \eqref{eq:minimax-arm-i}, one can easily prove the problem-independent regret bound of Thompson Sampling by setting
$\epsilon=\Delta_i/2$ and summing up over $i=1,\ldots,K$ \citep{agrawal2017near}. Similar techniques are also used in proving the  regret bound of UCB algorithms~\citep{lattimore2018bandit}.

Note that by the definition of $D_{is}$, $G_{is}(\epsilon)$ is a random variable depending on $\hat{\mu}_{is}$. For brevity, however, we do not explicitly indicate this dependency by writing $G_{is}(\epsilon)$ as $G_{is}(\epsilon,\hat{\mu}_{is})$; such shortened notations are also used in \cite{agrawal2017near,lattimore2018bandit}.


Though $G_{is}(\epsilon)$ is defined based on the clipped Gaussian distribution $D_{is}$, the right-hand side of \eqref{eq:minimax-arm-i} and \eqref{eq:asymptotically-arm-i} can be bounded in the same way for Gaussian distributions like in \citet{agrawal2017near}. We need some notations. Let $F'_{is}$ be the CDF of Gaussian distribution $\mathcal{N}(\hat{\mu}_{is},1/(\rho s))$ for any $s\geq 1$.  Let $G'_{is}(\epsilon)=1-F'_{is}(\mu_1-\epsilon)$. We have the following lemma.
\begin{lemma}
\label{lem:mian-bounding-I2}
Let $\rho\in(1/2,1)$ be a constant. Under the conditions in Theorem \ref{theorem:minimax-optimal}, for any $\epsilon>0$, there exists a universal constant $c>0$ such that:
 \begin{equation} \label{eq:c/eps^2}\small
\mathbb{E}\bigg[\sum_{s=1}^{T-1}\left(\frac{1}{G'_{1s}(\epsilon)}-1\right)\bigg] \leq \frac{c}{\epsilon^2}.
\end{equation}
\end{lemma}
Similar quantities are also bounded in \citet{agrawal2017near,lattimore2018bandit}, which are essential for proving the near optimal problem-independent regret bound for Thompson sampling. However, the upper bound in Lemma \ref{lem:mian-bounding-I2} is sharper than that in previous papers due to the scaling parameter $\rho$ we choose in our MOTS algorithm. In fact, the requirement $\rho\in(1/2,1)$ is necessary to obtain such an improved upper bound. In the next lemma, we will show that if we choose $\rho=1$ as is done in existing work, the second term in the right-hand side of \eqref{eq:minimax-arm-i} will have a nontrivial lower bound. 
\begin{lemma}[Lower Bound]
\label{lem:decom=false}
Assume $K\log T\leq \sqrt{T}$. If we set $\rho=1$, then there exists a bandit instance with $\Delta_i=2\sqrt{K\log T/T}$ for all $i\in \{2,3,\cdots,K\}$ such that for any $\epsilon>0$
\begin{align}
\label{eq:seps^2}
    \EE\bigg[ \frac{1}{G'_{1s}(\epsilon)}-1\bigg]\geq \frac{e^{-\frac{s\epsilon^2}{2}}}{s\epsilon^2} ,
\end{align}
and the decomposition in \eqref{eq:minimax-arm-i} will lead to
\begin{align*}
    K\Delta_i\cdot  \EE\bigg[\sum_{s=1}^{T-1} \bigg( \frac{1}{G'_{1s}(\Delta_i/2)}-1\bigg)\bigg]= \Omega(\sqrt{KT \log T}). 
\end{align*}
\end{lemma}
The above lemma shows that if we set $\rho=1$, 
the decomposition in \eqref{eq:minimax-arm-i} will lead to an unavoidable $\Omega(\sqrt{KT\log T})$ problem-independent regret. 
Combined with Lemma \ref{lem:mian-bounding-I2}, it indicates that our choice of $\rho\in(1/2,1)$ in MOTS is crucial to improve the previous analysis and obtain a better  regret bound.  When the reward distribution is Bernoulli, it is worth noting that \citet{agrawal2017near}  achieved an improved  regret $O(\sqrt{KT\log K})$ by using Gaussian priors. Meanwhile, they also proved that this regret bound is unimprovable for Thompson sampling using Gaussian priors, which leaves a gap in achieving the minimax optimal regret $O(\sqrt{KT})$. In the following proof of Theorem \ref{theorem:minimax-optimal}, we will show that the clipped Gaussian distribution suffices to close this gap and achieve the $O(\sqrt{KT})$ minimax regret. Moreover, in Theorem~\ref{theorem_J_distribution}, we will further show that MOTS-$\cJ$ can achieve the minimax optimal regret by using the Rayleigh distribution and does not need the requirement on the scaling parameter $\rho$, which is crucial in proving the asymptotic optimality simultaneously. 

Now, we are ready to prove the minimax optimality of MOTS.
\begin{proof}[Proof of Theorem \ref{theorem:minimax-optimal}]
Recall that $\hat{\mu}_{is}$ is the average reward of arm $i$ when it has been played $s$ times. We define $\Delta$ as follows:
\begin{equation}
\label{eq:def-Delta}
\Delta=\mu_1-\min_{1\leq s\leq T}\Bigg\{\hat{\mu}_{1s}+\sqrt{\frac{\alpha}{s}\log^{+}\bigg(\frac{T}{sK}\bigg)}\Bigg\}.
\end{equation}
The regret of Algorithm \ref{alg:mots} can be decomposed as follows. 
\begin{align}
     \label{eq:minimax-decompose-}
R_{\mu}(T)& =\sum_{i:\Delta_i>0}\Delta_i\mathbb{E}[T_i(T)] \notag \\
    &\leq \mathbb{E}[2T\Delta]+\EE\left[\sum_{i:\Delta_i>2\Delta}\Delta_iT_i(T) \right]  \notag \\
    &\leq \mathbb{E}[2T\Delta]+8\sqrt{KT}  +\EE\left[\sum_{i:\Delta_i>\max\{2\Delta, 8\sqrt{K/T}\}}\Delta_i T_i(T)\right].
\end{align}
The first term in \eqref{eq:minimax-decompose-} can be bounded as:
\begin{equation} \label{eq:minimax-TDelta}\small
\begin{split}
    \mathbb{E}[2T\Delta] =2T \int_{0}^{\infty} \mathbb{P}(\Delta\geq x) \dd x \leq  2T \int_{0}^{\infty} \min\bigg\{1, \frac{15K}{Tx^2}\bigg\} \dd x = 4\sqrt{15KT},
      \end{split}
\end{equation}
where the inequality comes from Lemma \ref{lemma:high_prob_all_positive} since 
\begin{align}
    &\PP\Bigg(\mu_1-\min_{1\leq s\leq T}\Bigg\{\hat{\mu}_{1s}+\sqrt{\frac{\alpha}{s}\log^+\bigg(\frac{T}{sK}\bigg)}\Bigg\}\geq x\Bigg) \notag\\
    =&\PP\Bigg(\exists 1\leq s\leq T: \mu_1-\hat{\mu}_{1s}-\sqrt{\frac{\alpha}{s}\log^+\bigg(\frac{T}{sK}\bigg)}-x\geq 0 \Bigg).\notag
\end{align}
Define set $S=\{i: \Delta_i>\max\{2\Delta,8\sqrt{K/T} \}\}$. Now we focus on term $\sum_{i\in S}\Delta_i T_i(T)$. Note that the update rules of Algorithm~\ref{alg:mots} ensure $D_i(t+1)=D_i(t)$ ($t\geq K+1$) whenever $A_t\neq i$.  We define 
\begin{align}
  \tau_{is}=\hat{\mu}_{is}+\sqrt{\frac{\alpha}{s}\log^{+} \bigg(\frac{T}{sK}\bigg)}.
\end{align}
By the definition in \eqref{eq:shrink_range}, we have $\tau_{is}=\tau_{i}(t)$ when $T_i(t)=s$. From the definition of $\Delta$ in \eqref{eq:def-Delta}, for $i\in S$, we have
\begin{align} \label{eq:best-arm-lower-bound}
     \tau_{1s}=\hat{\mu}_{1s}+\sqrt{\frac{\alpha}{s}\log^{+} \bigg(\frac{T}{sK} \bigg)}  \geq \mu_1-\Delta \geq \mu_1-\frac{\Delta_i}{2}.
\end{align}
Recall the definition of $D_{1s}$. Let $\theta_{1s}$ be a sample from the clipped distribution $D_{1s}$. As mentioned in Section~\ref{sec:sample_theta}, we obtain $\theta_{1s}$ with the following procedure. We first sample $\tilde{\theta}_{1s}$  from distribution $\mathcal{N}(\hat{\mu}_{1s},1/(\rho s))$. If $\tilde{\theta}_{1s}< \tau_{1s}$, we set $\theta_{1s}=\tilde{\theta}_{1s}$; otherwise, we set $\theta_{1s}=\tau_{1s}$. \eqref{eq:best-arm-lower-bound} implies  that $\mu_1-\Delta_i/2\leq\tau_{1s}$, where $\tau_{1s}$ is the boundary for clipping. Therefore, $\PP(\tilde{\theta}_{1s}\geq \mu_1-\Delta_i/2)=\PP({\theta}_{1s} \geq \mu_1-\Delta_i/2)$. By definition, $F'_{is}$ is the CDF of $\mathcal{N}(\hat{\mu}_{is},1/(\rho s))$ and $G'_{is}(\epsilon)=1-F'_{is}(\mu_1-\epsilon)$. Therefore, for any $i\in S$, $G_{1s}(\Delta_i/2)=\PP({\theta}_{1s} \geq \mu_1-\Delta_i/2)= \PP(\tilde{\theta}_{1s}\geq \mu_1-\Delta_i/2)=G'_{1s}(\Delta_i/2)$. 

Using~\eqref{eq:minimax-arm-i} of Lemma~\ref{lem:thom-bound} and setting  $\epsilon={\Delta_i}/{2}$, for any $i\in S$, we have
\begin{equation}\label{eq:minimax-Deltai/2} 
\begin{split}
\Delta_i \mathbb{E}[T_i(T)]&\leq \Delta_i+ \Delta_i \cdot \mathbb{E}\bigg[\sum_{t=K+1}^{T-1}\ind\{ A_t=i,E_{i}^{c}(t)\}\bigg]+ \Delta_i\cdot \mathbb{E}\bigg[\sum_{s=1}^{T-1}\left(\frac{1}{G_{1s}(\Delta_i/2)}-1\right)\bigg]\\
& = \Delta_i+\underbrace{\Delta_i \cdot \mathbb{E}\bigg[\sum_{t=K+1}^{T-1}\ind\{ A_t=i,E_{i}^{c}(t)\}\bigg]}_{I_1} +\underbrace{\Delta_i\cdot \mathbb{E}\bigg[\sum_{s=1}^{T-1}\left(\frac{1}{G'_{1s}(\Delta_i/2)}-1  \right)\bigg]}_{I_2} .
\end{split}
\end{equation} 
\textbf{Bounding term $I_1$:} Note that
\begin{align*}
     E_{i}^c(t) & =\bigg\{\theta_i(t)> \mu_1-\frac{\Delta_i}{2} \bigg\} \subseteq \Bigg\{\hat{\mu}_{i}(t)  +\sqrt{\frac{\alpha}{T_i(t)}\log^{+}\bigg( \frac{T}{KT_{i}(t)} \bigg)}>\mu_1-\frac{\Delta_i}{2}  \Bigg\}. 
\end{align*}
We define the following notation:
\begin{equation}
\small\kappa_i=\sum_{s=1}^{T}\ind \Bigg \{\hat{\mu}_{is}+\sqrt{\frac{\alpha}{s}\log^+\bigg(\frac{T}{sK}\bigg)}> \mu_1-\frac{\Delta_i}{2} \Bigg \},
\end{equation}
which immediately implies that
{ \begin{align} \label{eq:minimax-final}
  I_1= \Delta_i \cdot \mathbb{E}\bigg[\sum_{t=K+1}^{T-1}\ind\{ A_t=i,E_{i}^{c}(t)\}\bigg] \leq\Delta_i\EE[\kappa_i].
\end{align}}%
To further bound \eqref{eq:minimax-final}, we have
\begin{align}\label{eq:minimax-kappa}
\Delta_i\EE[\kappa_i]&=
  \Delta_i\EE\Bigg[\sum_{s=1}^{T}\ind \Bigg \{\hat{\mu}_{is}+\sqrt{\frac{\alpha}{s}\log^+\bigg(\frac{T}{sK}\bigg)}> \mu_1-\frac{\Delta_i}{2} \Bigg \}\Bigg] \notag \\
   & \leq \Delta_i \sum_{s=1}^{T}\PP \Bigg \{\hat{\mu}_{is}-\mu_i+\sqrt{\frac{\alpha}{s}\log^+\bigg(\frac{T}{sK}\bigg)}> \frac{\Delta_i}{2} \Bigg \} \notag \\
   & \leq \Delta_i+\frac{12}{\Delta_i}+\frac{4\alpha}{\Delta_i}\Bigg( \log^+ \bigg( \frac{T\Delta_i^2}{4K} \bigg)+\sqrt{2\alpha\pi\log^+\bigg(\frac{T\Delta_i^2}{4K}}\bigg) \Bigg),
\end{align}
where the first inequality is due to the fact that $\mu_1-\mu_i=\Delta_i$ and the second one is by Lemma~\ref{lem:sumTpro}.
It can be verified that $h(x)= x^{-1}\log^{+} (ax^2)$ is monotonically decreasing for $x\geq e/\sqrt{a}$ and any $a>0$. Since $\Delta_i\geq8\sqrt{K/T}> e/\sqrt{T/(4K)}$, we have $\log(T\Delta_i^2/(4K))/\Delta_i\leq \sqrt{T/K}$. 
Plugging this into \eqref{eq:minimax-kappa},  we have $\EE[\Delta_i\kappa_i] = O(\sqrt{T/K}+\Delta_i)$.


\noindent\textbf{Bounding term $I_2$:} applying Lemma~\ref{lem:mian-bounding-I2}, we immediately obtain
\begin{equation}
\label{eq:results-bounding-I_2}
    I_2=\Delta_i \mathbb{E}\bigg[\sum_{s=1}^{T-1}\left(\frac{1}{G'_{1s}(\Delta_i/2)}-1\right)\bigg] = O\bigg(\sqrt{\frac{T}{K}}\bigg).
\end{equation}
Substituting \eqref{eq:minimax-TDelta}, \eqref{eq:minimax-Deltai/2}, \eqref{eq:minimax-kappa}, and \eqref{eq:results-bounding-I_2} into \eqref{eq:minimax-decompose-}, we complete the proof of Theorem~\ref{theorem:minimax-optimal}. 
\end{proof}

\section{Experiments}
\begin{figure*}[ht]
\subfigure[$K=50$, $\epsilon=0.2$]{\label{fig:k50-02}\includegraphics[height=0.245\textwidth]{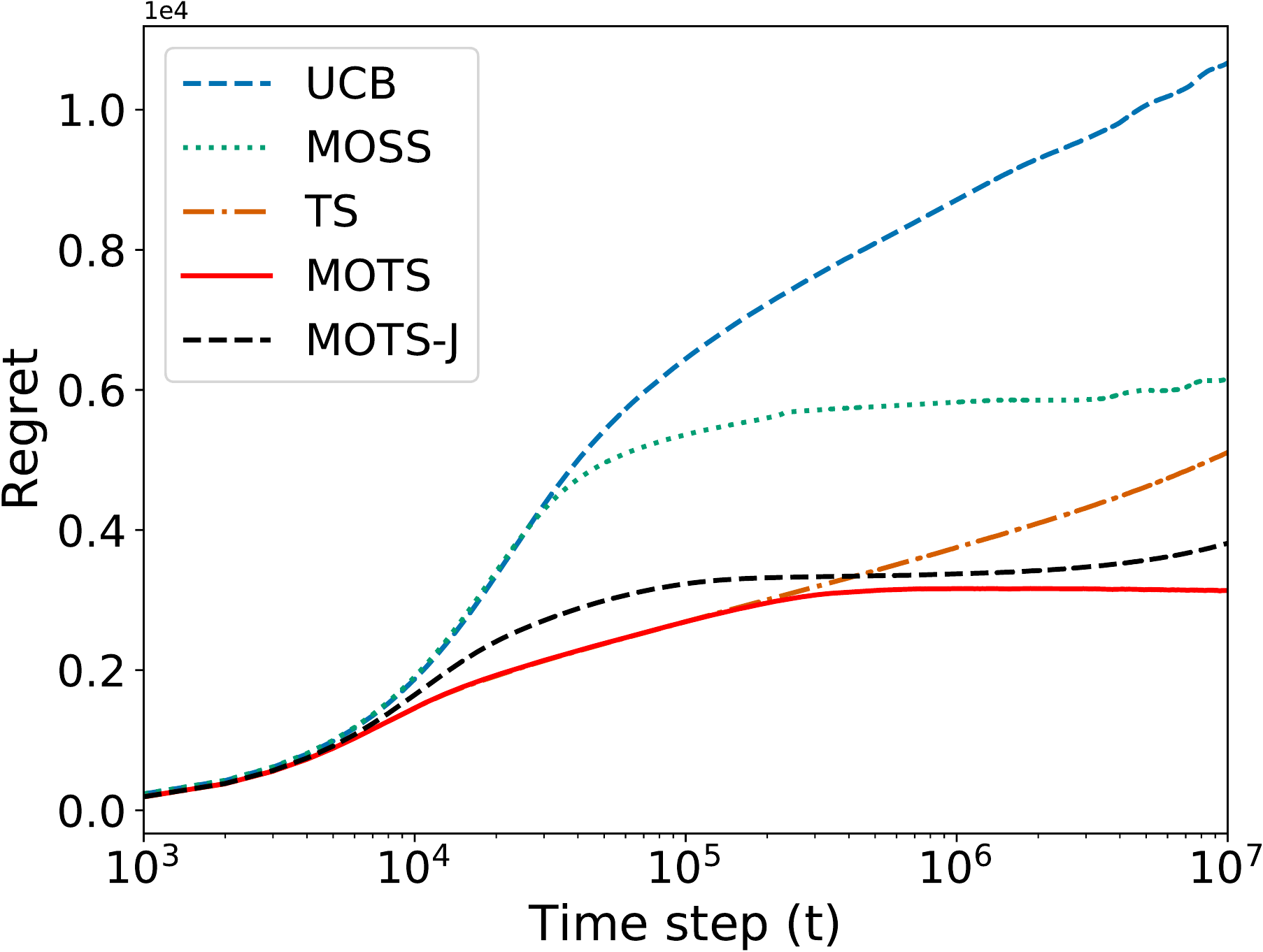}}
\subfigure[$K=50$, $\epsilon=0.1$]{\label{fig:k50-01}\includegraphics[height=0.245\textwidth]{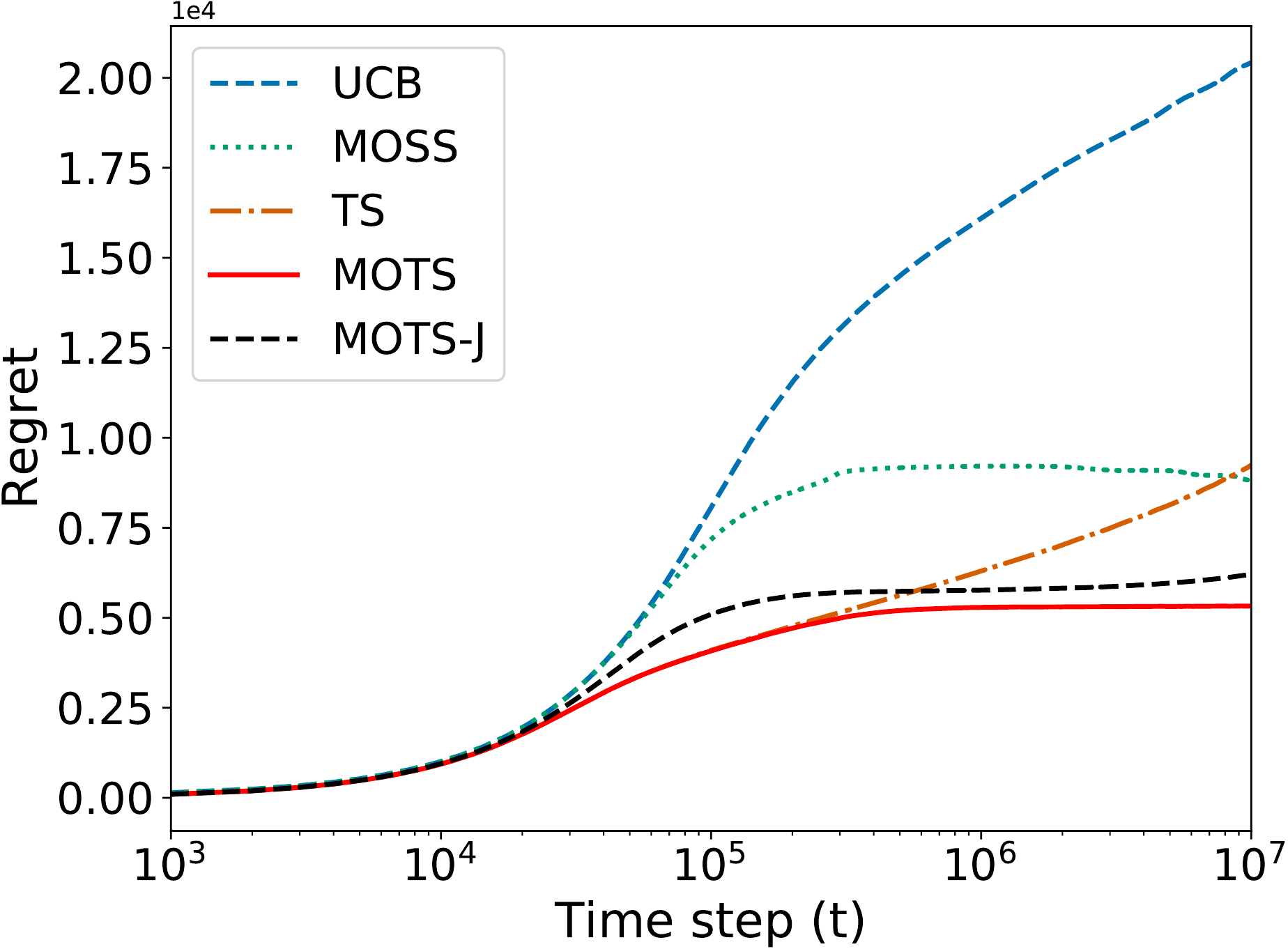}}
\subfigure[$K=50$, $\epsilon=0.05$]{\label{fig:k50-005}\includegraphics[height=0.245\textwidth]{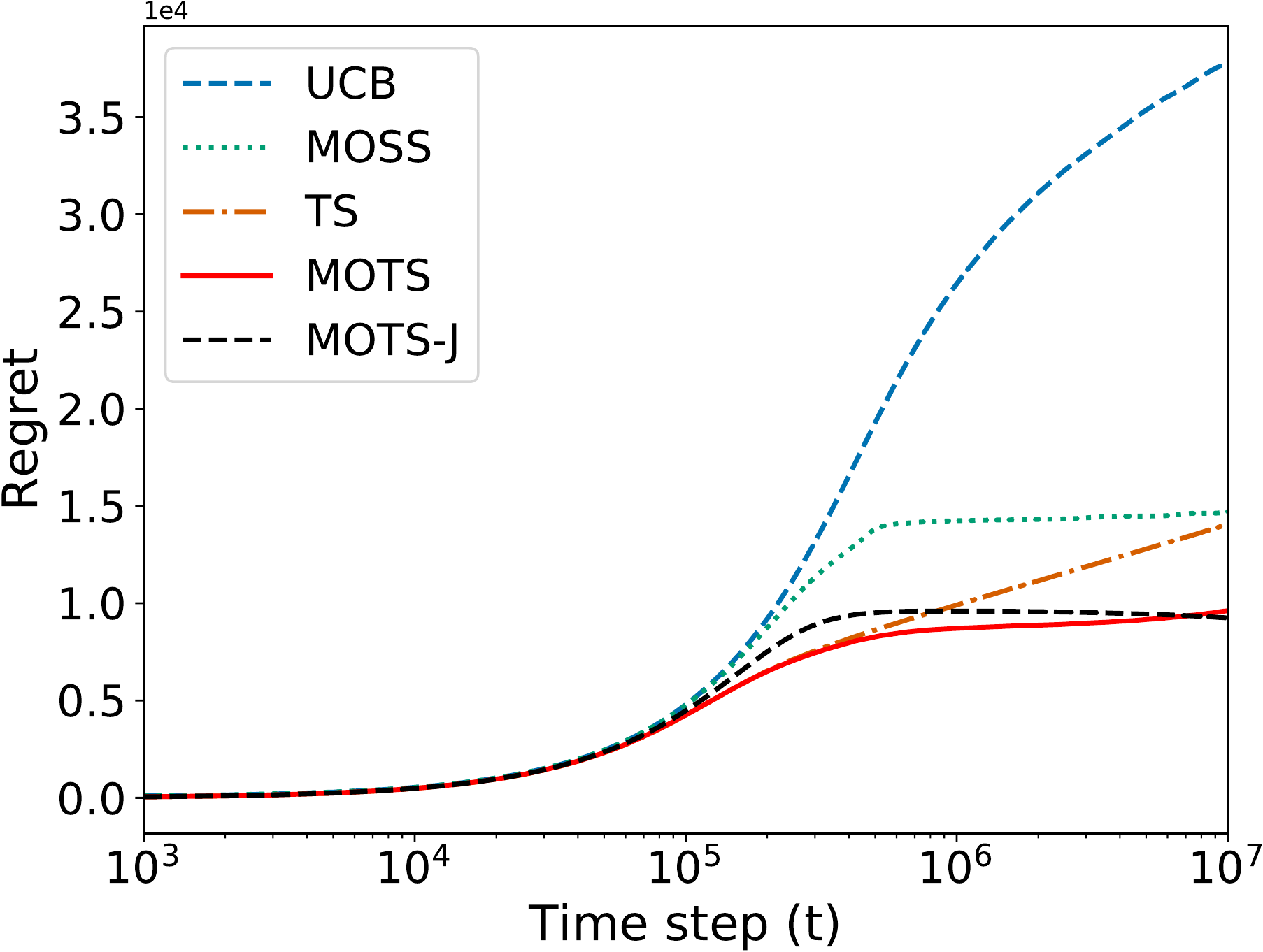}}
\caption{The regret for different algorithms with $K=50$ and $\epsilon\in \{0.2,0.1,0.05\}$. The experiments are averaged over $6000$ repetitions.}
\vspace{-2mm}
\label{fig:k50}
\end{figure*}
\begin{figure}[h]
\subfigure[$K=100$, $\epsilon=0.2$]{\label{fig:k10-02}\includegraphics[height=0.245\textwidth]{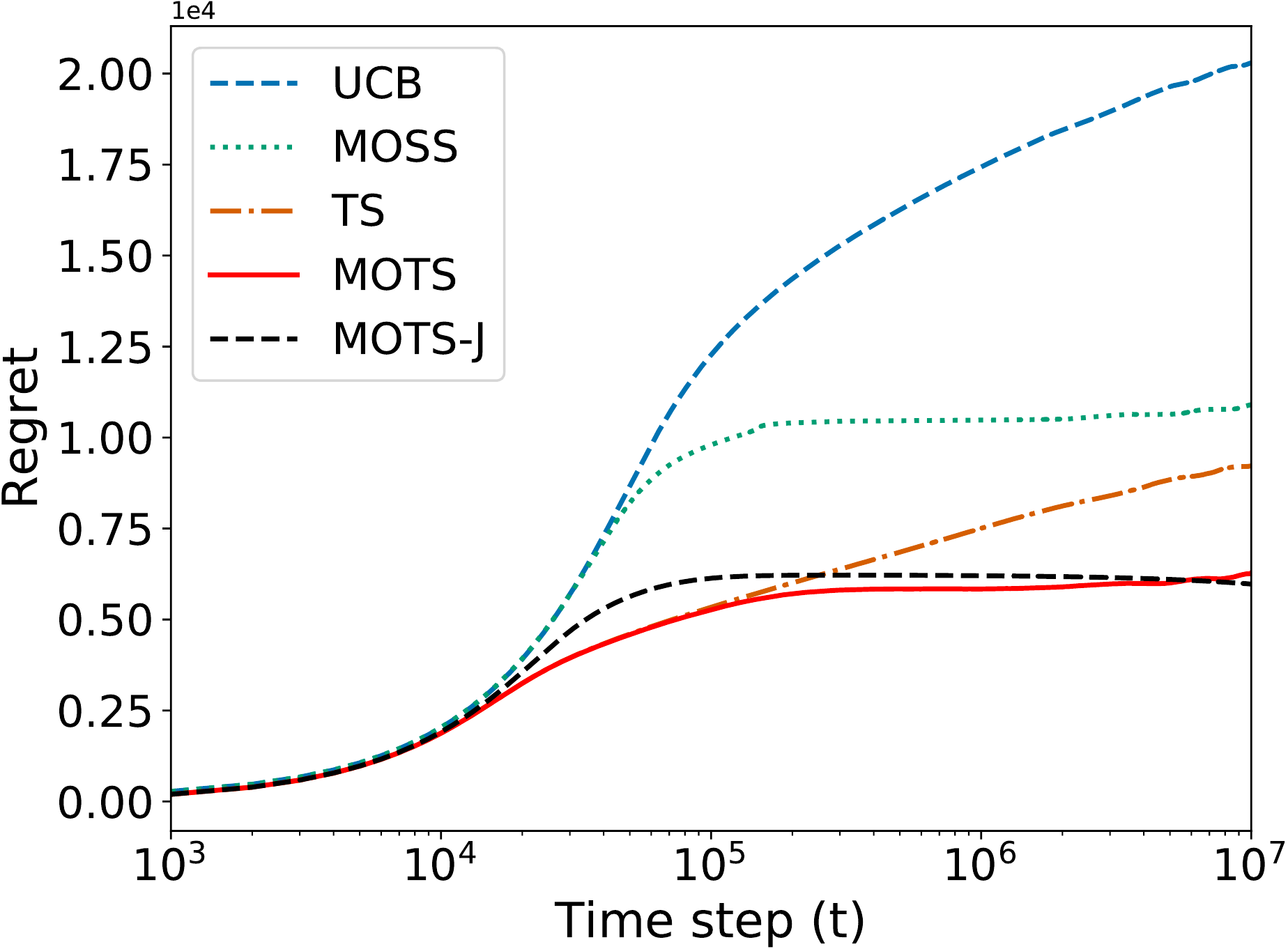}}
\subfigure[$K=100$, $\epsilon=0.1$]{\label{fig:k10-01}\includegraphics[height=0.245\textwidth]{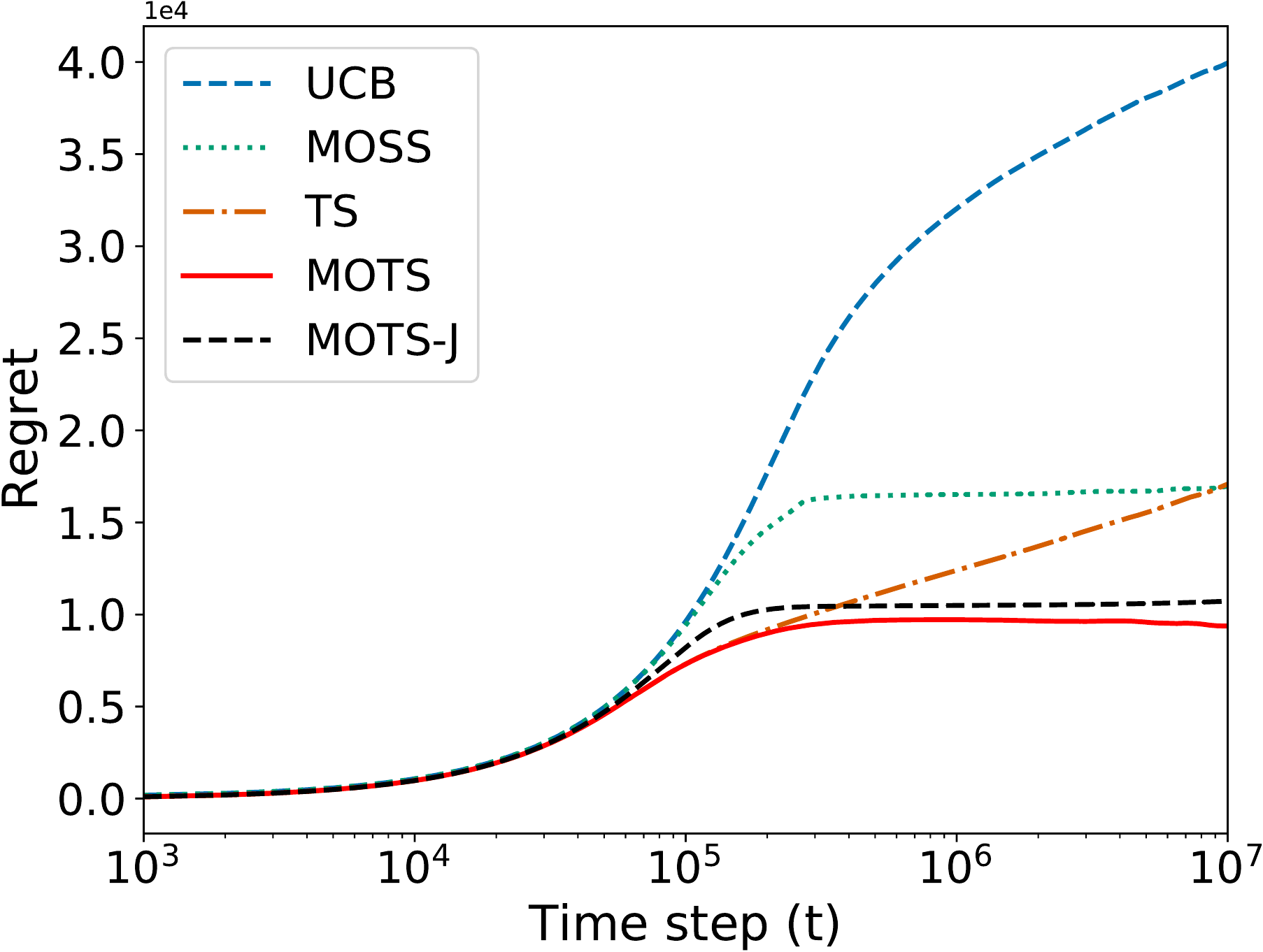}}
\subfigure[$K=100$, $\epsilon=0.05$]{\label{fig:k10-005}\includegraphics[height=0.245\textwidth]{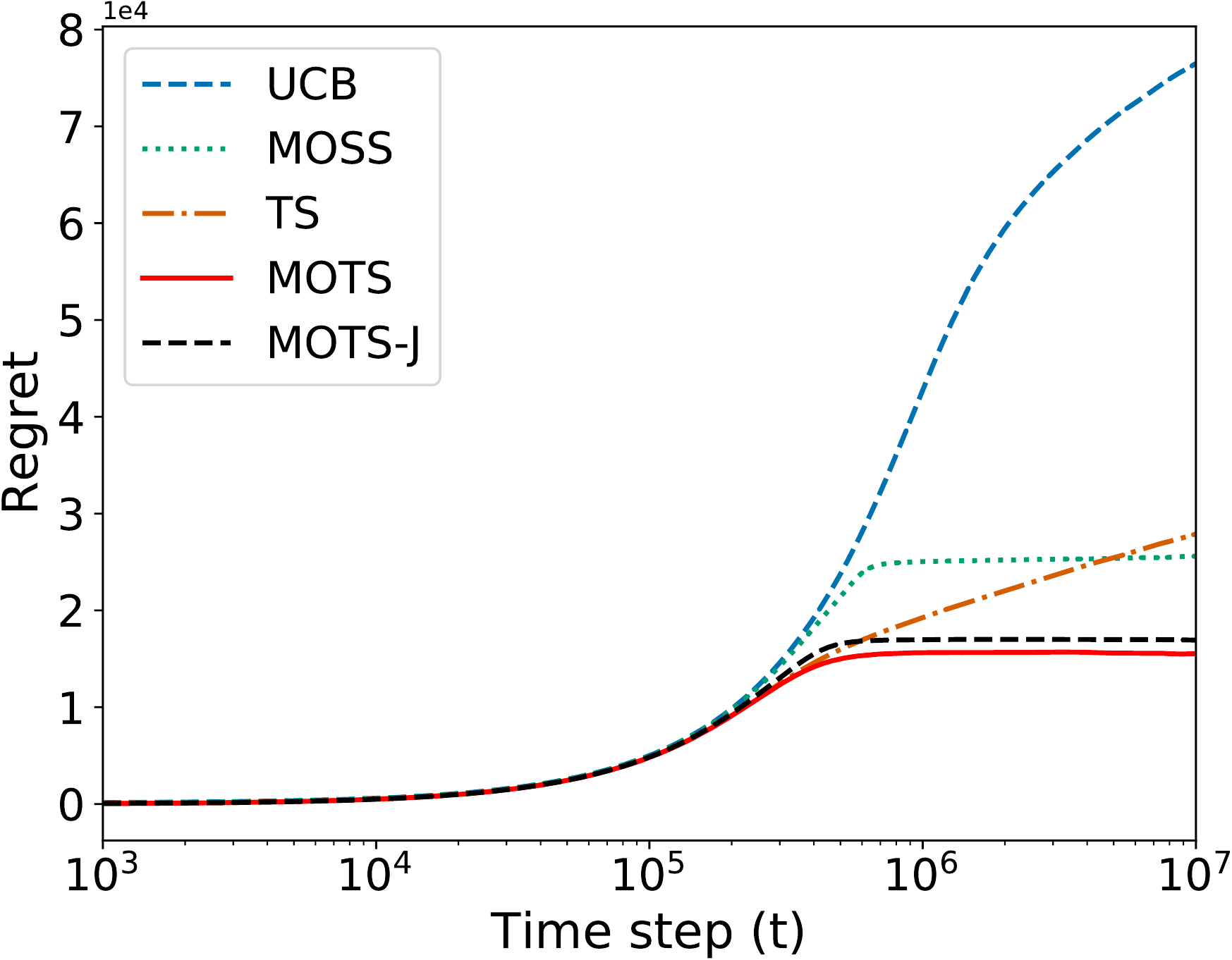}}
\caption{The regret for different algorithms with $K=100$ and $\epsilon\in \{0.2,0.1,0.05\}$. The experiments are averaged over $6000$ repetitions.} 
\vspace{-2mm}
\label{fig:k10}
\end{figure}

In this section, we experimentally compare our proposed algorithms MOTS and MOTS-$\cJ$ with existing algorithms for multi-armed bandit problems with Gaussian rewards. Baseline algorithms include MOSS~\citep{audibert2009minimax}, 
UCB~\citep{katehakis1995sequential}, 
and Thompson sampling with Gaussian priors (TS for short)~\citep{agrawal2017near}. We consider two settings: $K=50$ and $K=100$, where $K$ is the number of arms. In both settings, each arm follows an independent Gaussian distribution.
The best arm has expected reward  1 and variance 1, while the other $K-1$ arms have expected reward  $1-\epsilon$ and variance 1.
We vary $\epsilon$ with values $0.2, 0.1, 0.05$ in different experiments. The total number of time steps
$T$ is set to $10^7$. In all experiments, the parameter $\rho$ for MOTS defined in Section \ref{sec:sample_theta} is set to $0.9999$. 
Since we focus on Gaussian rewards, we set $\alpha=2$ in  \eqref{eq:shrink_range} for both MOTS and MOTS-$\cJ$.  

For MOTS-$\cJ$, we need to sample instances from  distribution $\cJ(\mu,\sigma^2)$, of which the PDF is defined in \eqref{eq:def_pdf_of_J}. To sample from $\cJ$, we use the well known inverse transform sampling technique by
 first computing the corresponding inverse CDF,
and then uniformly choosing a random number in $[0,1]$, which is then used to calculate the random number sampled from $\cJ(\mu,\sigma^2)$.

In the setting of $K=50$, Figures \ref{fig:k50-02}, \ref{fig:k50-01}, and \ref{fig:k50-005} report the regrets of all algorithms when $\epsilon$ is 0.2, 0.1, 0.05 respectively.
For all $\epsilon$ values, MOTS consistently outperforms the baselines for all time step $t$, and  MOTS-$\cJ$ outperforms the baselines especially when  $t$ is large.
For instance, in Figure \ref{fig:k50-005}, when time step $t$ is $T=10^7$, the regret of MOTS and MOTS-$\cJ$ are 9615 and 9245 respectively, while the regrets of TS, MOSS, and UCB are 14058, 14721, and 37781 respectively.


In the setting of $K=100$, Figures \ref{fig:k10-02}, \ref{fig:k10-01}, and \ref{fig:k10-005} report the regrets of MOTS, MOTS-$\cJ$, MOSS, TS, and UCB when $\epsilon$ is 0.2, 0.1, 0.05 respectively. Again, for all $\epsilon$ values, when varying the time step $t$, MOTS consistently has the smallest regret,  outperforming all baselines, and MOTS-$\cJ$ outperforms all baselines especially when $t$ is large. 


In summary, our algorithms  consistently outperform TS, MOSS, and UCB when varying $\epsilon$, $K$, and $t$.

\section{Conclusion and Future Work}
We solved the open problem on the minimax optimality for Thompson sampling \citep{li2012open}. We proposed the MOTS algorithm and proved that it achieves the minimax optimal regret $O(\sqrt{KT})$ when rewards are generated from subGaussian distributions. In addition, we propose a variant of MOTS called MOTS-$\cJ$ that simultaneously achieves the minimax and asymptotically optimal regret for $K$-armed bandit problems when rewards are generated from Gaussian distributions.
Our experiments demonstrate the superior performances of  MOTS and MOTS-$\cJ$ compared with the state-of-the-art bandit algorithms.

Interestingly, our experimental results show that the performance of MOTS is never worse than that of MOTS-$\cJ$. Therefore, it would be an interesting future direction to investigate whether the proposed MOTS with clipped Gaussian distributions can also achieve both minimax and asymptotic optimality for multi-armed bandits.


\appendix

\section{Proofs of Theorems}
\label{theorem:proof}
In this section, we provide the proofs of Theorems \ref{theorem:asymptotic-optimal}, \ref{theorem-new3} and \ref{theorem_J_distribution}.
\subsection{Proof of Theorem \ref{theorem:asymptotic-optimal}}\label{sec:proof_of_asy_subg}
To prove Theorem \ref{theorem:asymptotic-optimal}, we need the following technical lemma.
\begin{lemma}\label{lem:asym1} 
For any $\epsilon_T>0$, $\epsilon>0$ that satisfies $\epsilon+\epsilon_T<\Delta_i$, it holds that
\begin{align*}
  \mathbb{E}\bigg[\sum_{s=1}^{T-1}\ind\{G'_{is}(\epsilon)>1/T\}\bigg]\leq 1+ \frac{2}{\epsilon_T^2}+\frac{2\log T}{\rho(\Delta_i-\epsilon-\epsilon_T)^2}.
\end{align*} 
\end{lemma}

\begin{proof}[Proof of Theorem \ref{theorem:asymptotic-optimal}]
Let $Z(\epsilon)$ be the following event
\begin{align}\label{eq:event_z_eps}
Z(\epsilon)=\bigg\{\forall s\in [T]: \hat{\mu}_{1s}+\sqrt{\frac{\alpha}{s}\log^{+}\Big(\frac{T}{sK}\Big)}\geq \mu_1-\epsilon\bigg\}.
\end{align}
For any arm $i\in[K]$, we have
\begin{align}\label{eq:asym-main}
   \mathbb{E}[T_i(T)]&\leq\EE[T_i(T) \mid Z(\epsilon)]\PP(Z(\epsilon))+T(1-\PP[Z(\epsilon)])\notag\\
  & \leq 2+ \mathbb{E}\bigg[\sum_{s=1}^{T-1}\left(\frac{1}{G_{1s}(\epsilon)}-1\right) \bigg| Z(\epsilon) \bigg] +T(1-\PP[Z(\epsilon)])+{ \mathbb{E}\bigg[\sum_{s=1}^{T-1}\ind\{G_{is}(\epsilon)>1/T\}\bigg]} \notag\\ 
  & \leq 2+  \mathbb{E}\bigg[\sum_{s=1}^{T-1}\left(\frac{1}{G'_{1s}(\epsilon)}-1\right)\bigg] +T(1-\PP[Z(\epsilon)])+\mathbb{E}\bigg[\sum_{s=1}^{T-1}\ind\{G_{is}(\epsilon)>1/T\}\bigg] \notag \\
 & \leq 2+  \mathbb{E}\bigg[\sum_{s=1}^{T-1}\left(\frac{1}{G'_{1s}(\epsilon)}-1\right)\bigg] +T(1-\PP[Z(\epsilon)])+\mathbb{E}\bigg[\sum_{s=1}^{T-1}\ind\{G'_{is}(\epsilon)>1/T\}\bigg],
\end{align}
where the second inequality is due to \eqref{eq:asymptotically-arm-i} in Lemma \ref{lem:thom-bound}, the third inequality is due to the fact that conditional on event $Z(\epsilon)$ defined in \eqref{eq:event_z_eps} we have $G_{1s}(\epsilon)=G'_{1s}(\epsilon)$, and the last inequality is due to the fact that $G_{is}(\epsilon)=G'_{is}(\epsilon)$  for
\begin{equation}
\hat{\mu}_{is}+\sqrt{\frac{\alpha}{s}\log^{+}\Big(\frac{T}{sK}\Big)}\geq \mu_1-\epsilon,
\end{equation}
and $G_{is}(\epsilon)=0\leq G'_{is}(\epsilon)$ for 
\begin{equation}
\hat{\mu}_{is}+\sqrt{\frac{\alpha}{s}\log^{+}\Big(\frac{T}{sK}\Big)}< \mu_1-\epsilon.
\end{equation}
Let $\epsilon=\epsilon_T=1/\log \log T$. Applying Lemma~\ref{lemma:high_prob_all_positive}, we have
\begin{align}\label{eq:asy1}
    T(1-\PP[Z(\epsilon)]) & \leq  T\cdot \frac{15K}{T\epsilon^2}  \leq 15 K (\log \log T)^2.
\end{align}
Using Lemma~\ref{lem:mian-bounding-I2}, we have
\begin{align}\label{asy2}
\mathbb{E}\bigg[\sum_{s=1}^{T-1}\left(\frac{1}{G'_{1s}(\epsilon)}-1\right)\bigg]\leq O( (\log \log T)^2).
\end{align}
Furthermore using Lemma~\ref{lem:asym1}, we obtain
\begin{align}\label{asy3}
    \mathbb{E}\bigg[\sum_{s=1}^{T-1}\ind\{G'_{is}(\epsilon)>1/T\} \bigg] 
    &\leq 1+ 2(\log \log T)^2 + \frac{2\log T}{\rho(\Delta_i-2/\log \log T)^2}.
\end{align}
Combine~\eqref{eq:asym-main}, \eqref{eq:asy1}, \eqref{asy2} and \eqref{asy3} together, we finally obtain
\begin{equation}\label{eq:proof_of_asy_subg_result}
  \lim_{T\rightarrow \infty}  \frac{ \mathbb{E}[\Delta_i T_i(T)]}{\log T} = \frac{2}{\rho \Delta_i}.
\end{equation}
This completes the proof for the asymptotic regret.
\end{proof}

\subsection{Proof of Theorem~\ref{theorem-new3}}

In the proof of Theorem~\ref{theorem:minimax-optimal} (minimax optimality), we need to bound $I_2$ as in \eqref{eq:results-bounding-I_2}, which calls the conclusion of Lemma~\ref{lem:mian-bounding-I2}. However, the value of $\rho$ is a fixed constant in  Lemma~\ref{lem:mian-bounding-I2}, which thus is absorbed into the constant $c$. In order to show the dependence of the regret on $\rho$ chosen as in Theorem  \ref{theorem-new3}, we need to replace Lemma~\ref{lem:mian-bounding-I2} with the following variant.  
\begin{lemma}
\label{lem:bounding-I_2-notfixed}
Let $\rho=1-\sqrt{40/\ilog^{(m)} (T)}$. Under the conditions in Theorem \ref{theorem-new3}, there exists a universal constant $c>0$ such that 
\begin{align}
    \EE\bigg[\sum_{s=1}^{T-1} \bigg(\frac{1}{G'_{1s}(\epsilon)-1} \bigg) \bigg] \leq \frac{c\ilog^{(m-1)}(T)}{\epsilon^2}.
\end{align}
\end{lemma}

\begin{proof}[Proof of Theorem \ref{theorem-new3}]
From Lemma~\ref{lem:bounding-I_2-notfixed}, we immediately obtain
\begin{equation}
\label{eq:results-bounding-I_2-notfix}
    I_2=\Delta_i \mathbb{E}\bigg[\sum_{s=1}^{T-1}\left(\frac{1}{G'_{1s}(\Delta_i/2)}-1\right)\bigg]  \leq O\bigg(\ilog^{(m-1)}(T)\sqrt{\frac{T}{K}}+\Delta_i\bigg),
\end{equation}
where $I_2$ is defined the same as in \eqref{eq:results-bounding-I_2}.    
Note that the above inequality only changes the result in \eqref{eq:results-bounding-I_2} and the rest of the proof of Theorem~\ref{theorem:minimax-optimal} remains the same. Therefore, substituting \eqref{eq:minimax-TDelta}, \eqref{eq:minimax-Deltai/2}, \eqref{eq:minimax-kappa} and \eqref{eq:results-bounding-I_2-notfix} back into \eqref{eq:minimax-decompose-}, 
we have
\begin{equation}
    R_{\mu}(T)\leq O\bigg(\sqrt{KT}\ilog^{(m-1)}(T)+\sum_{i=2}^K\Delta_i\bigg).
\end{equation}

For the asymptotic regret bound, the proof is the same as that of Theorem~\ref{theorem:asymptotic-optimal} presented in Section \ref{sec:proof_of_asy_subg} since we have explicitly kept the dependence of $\rho$ during the proof. Note that $\rho=1-\sqrt{40/\ilog^{(m) }(T)}\rightarrow1$ when $T\rightarrow\infty$. Combining this with \eqref{eq:proof_of_asy_subg_result}, we have proved the asymptotic regret bound in Theorem~\ref{theorem-new3}.
 \end{proof}

\subsection{Proof of Theorem~\ref{theorem_J_distribution}}

\begin{proof}
For the ease of exposition, we follow the same notations used in Theorem~\ref{theorem:minimax-optimal} and~\ref{theorem:asymptotic-optimal}, except that we redefine two notations:  let $F'_{is}$ be the CDF of $\mathcal{J}(\hat{\mu}_{is},1/ s)$ for any $s\geq 1$ and  $G'_{is}(\epsilon)=1-F'_{is}(\mu_1-\epsilon)$, since  Theorem~\ref{theorem_J_distribution} uses clipped $\cJ$ distribution. 

In Theorem~\ref{theorem_J_distribution}, the proof of the minimax optimality is similar to that of Theorem~\ref{theorem:minimax-optimal} and the proof of asymptotic optimality is similar to that of Theorem~\ref{theorem:asymptotic-optimal}. 
We first focus on the minimax optimality. Note that in Theorem~\ref{theorem_J_distribution}, we assume $\alpha\geq 2$ while we have $\alpha\geq 4$ in Theorem~\ref{theorem:minimax-optimal}. Therefore, we need to replace the concentration property in Lemma \ref{lemma:high_prob_all_positive} by the following lemma which gives a sharper bound.
\begin{lemma}\label{lemma:high_prob_all_positive1}
Let  $X_1, X_2,\cdots$ be independent Gaussian random variables with zero mean and variance 1. Denote $\hat{\beta}_t=1/t\sum_{s=1}^{t}X_s$. Then  for $\alpha\geq 2$ and any $\Delta>0$, 
\begin{equation}
    \begin{split}
    \mathbb{P}\Bigg(\exists\ s\geq 1: \hat{\beta}_s+& \sqrt{\frac{\alpha}{s}\log^+\bigg( \frac{T}{sK }\bigg)}+\Delta\leq 0\Bigg)\leq \frac{4K}{T\Delta^2}.\\
    \end{split}
    \end{equation}
\end{lemma}

In the proof of Theorem~\ref{theorem:minimax-optimal} (minimax optimality), we need to bound $I_2$ as in \eqref{eq:results-bounding-I_2}, which calls the conclusion of Lemma~\ref{lem:mian-bounding-I2}, whose proof depends on the fact that $\rho<1$. In contrast, in Theorem~\ref{theorem_J_distribution}, we do not have the parameter $\rho$.
Therefore, we need to replace Lemma~\ref{lem:mian-bounding-I2} with the following variant. 

\begin{lemma}\label{lem:mian-bounding-I2-gaussian}
 Under the conditions in Theorem \ref{theorem_J_distribution}, there exists a universal constant $c>0$ such that:
 \begin{equation} \label{eq:the3-c/eps^2}\small
\mathbb{E}\bigg[\sum_{s=1}^{T-1}\left(\frac{1}{G'_{1s}(\epsilon)}-1\right)\bigg] \leq \frac{c}{\epsilon^2}.
\end{equation}
\end{lemma}
From Lemma~\ref{lem:mian-bounding-I2-gaussian}, we immediately obtain
\begin{equation}
\label{eq:results-bounding-I_2-notfix-the3}
    I_2=\Delta_i \mathbb{E}\bigg[\sum_{s=1}^{T-1}\left(\frac{1}{G'_{1s}(\Delta_i/2)}-1\right)\bigg]  \leq O\bigg(\sqrt{\frac{T}{K}}+\Delta_i\bigg).
\end{equation}

The rest of the proof for minimax optimality remains the same as that in Theorem~\ref{theorem:minimax-optimal}. Substituting \eqref{eq:minimax-TDelta}, \eqref{eq:minimax-Deltai/2}, \eqref{eq:minimax-kappa} and \eqref{eq:results-bounding-I_2-notfix-the3} back into \eqref{eq:minimax-decompose-}, 
we have
\begin{equation}
    R_{\mu}(T)\leq O\bigg(\sqrt{KT}+\sum_{i=2}^K\Delta_i\bigg).
\end{equation}
For the asymptotic regret bound, we will follow the proof of Theorem~\ref{theorem:asymptotic-optimal}. Note that Theorem~\ref{theorem:asymptotic-optimal} calls the conclusions of Lemmas~\ref{lemma:high_prob_all_positive}, \ref{lem:mian-bounding-I2} and \ref{lem:asym1}. 
To prove the asymptotic regret bound of Theorem \ref{theorem_J_distribution}, we replace Lemmas~\ref{lemma:high_prob_all_positive} and \ref{lem:mian-bounding-I2} by Lemmas~\ref{lemma:high_prob_all_positive1} and \ref{lem:mian-bounding-I2-gaussian} respectively, and further replace Lemma \ref{lem:asym1} by the following lemma. 

\begin{lemma}\label{lem:asym2} 
 Under the conditions in Theorem \ref{theorem_J_distribution}, for any $\epsilon_T>0$, $\epsilon>0$ that satisfies $\epsilon+\epsilon_T<\Delta_i$, it holds that
\begin{align*}
  \mathbb{E}\bigg[\sum_{s=1}^{T-1}\ind\{G'_{is}(\epsilon)>1/T\}\bigg]\leq 1+ \frac{2}{\epsilon_T^2}+\frac{2\log T}{(\Delta_i-\epsilon-\epsilon_T)^2}.
\end{align*} 
\end{lemma}

The rest of the proof is the same as that of Theorem~\ref{theorem:asymptotic-optimal},
and thus we omit it for simplicity. Note that in Theorem~\ref{theorem_J_distribution}, it does not have parameter $\rho$. Thus we have
\begin{equation}
    \lim_{T \rightarrow\infty}\frac{R_{\mu}(T)}{\log (T)}=\sum_{i:\Delta_i>0}\frac{2}{ \Delta_i},
\end{equation}
which completes the proof.
\end{proof}

\section{Proof of Supporting Lemmas}
\label{lemma:proof}
In this section, we prove the lemmas used in proving the main theories.

\subsection{Proof of Lemma~\ref{lemma:high_prob_all_positive}}
\begin{proof}
From Lemma 9.3 of~\cite{lattimore2018bandit}, we obtain 
\begin{align}
    \PP\bigg(\exists s\in [T]: \hat{\mu}_s+\sqrt{\frac{4}{s}\log^+\bigg(\frac{T}{sK} \bigg)}+\Delta\leq 0 \bigg) \leq \frac{15K}{T\Delta^2}.
\end{align}
Observing that for $\alpha\geq 4$ 
\begin{align}
    \sqrt{\frac{4}{s}\log^+\bigg(\frac{T}{sK} \bigg)}\leq \sqrt{\frac{\alpha}{s}\log^+\bigg(\frac{T}{sK} \bigg)},
\end{align}
Lemma~\ref{lemma:high_prob_all_positive} follows immediately.
\end{proof}

\subsection{Proof of Lemma~\ref{lem:sumTpro}}
We will need  the following property of subGaussian random variables.
\begin{lemma}[\citet{lattimore2018bandit}]
\label{lem:subguassian}
Assume that $X_1,\ldots,X_n$ are independent, $\sigma$-subGaussian random variables centered around $\mu$. Then for any $\epsilon>0$
\begin{align}
\label{eq:Gaussian-ineq}
    \mathbb{P}(\hat{\mu}\geq \mu+\epsilon)\leq \exp \bigg(
    -\frac{n\epsilon^2}{2\sigma^2}\bigg) \quad\text{and } \ \ \ \ \ \ \  \PP(\hat{\mu}\leq \mu-\epsilon)\leq \exp \bigg(
    -\frac{n\epsilon^2}{2\sigma^2}\bigg),
\end{align}
where $\hat{\mu}=1/n\sum_{t=1}^n X_t$.
\end{lemma}
\begin{proof}[Proof of Lemma~\ref{lem:sumTpro}]
Let $\gamma=\alpha\log^{+}(N\omega^2)/\omega^2$. Note that for $n\geq 1/w^2$, it holds that 
\begin{equation}
    \omega \sqrt{\frac{\gamma}{n}}=\sqrt{\frac{\alpha}{n}\log^{+}(N\omega^2)}\geq\sqrt{\frac{\alpha}{n}\log^{+}\Big(\frac{N}{n}\Big)}.
\end{equation}
Let $\gamma'= \max\{\gamma,1/w^2\}$. Therefore, we have
\begin{align}
    \sum_{n=1}^T \PP\bigg( \hat{\mu}_n + \sqrt{\frac{\alpha}{n}\log^{+}\Big(\frac{N}{n}\Big)} \geq \omega\bigg) 
    &\leq   \gamma'+ \sum_{n=\lceil\gamma  \rceil}^{T} \mathbb{P}\bigg(\hat{\mu}_{n}\geq \omega\bigg( 1-\sqrt{\frac{\gamma}{n}}\bigg)  \bigg) \notag\\
   &\leq    \gamma' + \sum_{n=\lceil\gamma  \rceil}^{\infty} \exp\bigg( -\frac{\omega^2( \sqrt{n}-\sqrt{\gamma})^2}{2}\bigg) \label{eq:sum_bound_stat1_hoeffding}\\
   &\leq   \gamma'+1+ \int_{\gamma}^{\infty} \exp\bigg( -\frac{\omega^2( \sqrt{x}-\sqrt{\gamma})^2}{2}\bigg) \dd x \notag\\
   &\leq  \gamma'+ 1+\frac{2}{\omega}\int_{0}^{\infty} \Big(\frac{y}{\omega}+\sqrt{\gamma}\Big)\exp(-y^2/2) \dd y \notag\\
   &\leq   \gamma'+1 +\frac{2}{\omega^2}+\frac{\sqrt{2\pi \gamma}}{\omega}, \label{eq:sum_bound_stat1_int}
\end{align}
where \eqref{eq:sum_bound_stat1_hoeffding} is the result of Lemma \ref{lem:subguassian} and \eqref{eq:sum_bound_stat1_int} is due to the fact that $\int_{0}^{\infty}y\exp(-y^2/2)\dd y=1$ and $\int_{0}^{\infty}\exp(-y^2/2)\dd y=\sqrt{2\pi}/2$. \eqref{eq:sum_bound_stat1_int} immediately implies the claim  of Lemma~\ref{lem:sumTpro}:
\begin{align}
    \sum_{n=1}^{T}\PP\bigg( \hat{\mu}_n + \sqrt{\frac{\alpha}{n}\log^{+}\Big(\frac{N}{n}\Big)} \geq \omega\bigg) 
    \leq & \gamma'+\sum_{n=\lceil \gamma \rceil}^{T} \PP \bigg(\hat{\mu}_n\geq \omega\bigg( 1- \sqrt{\frac{\gamma}{n}}\bigg) \bigg) \notag\\
     \leq & \gamma'+1+\frac{2}{\omega^2}+\frac{\sqrt{2\pi\gamma}}{\omega}.
\end{align}
Plugging $\gamma'\leq \alpha\log^{+}(N\omega^2)/\omega^2+1/w^2$ into the above inequality, we obtain
\begin{equation}
        \sum_{n=1}^T \mathbb{P}\bigg(\hat{\mu}_{n}+\sqrt{\frac{\alpha}{n}\log^{+}\bigg(\frac{N}{n}\bigg)} \geq \omega \bigg)\leq 1+\frac{\alpha\log^{+}({N}{\omega^2})}{\omega^2} +\frac{3}{\omega^2}+\frac{\sqrt{2\alpha\pi {\log^{+}({N}{\omega^2})}}}{\omega^2},
    \end{equation}
which completes the proof.
\end{proof}

\subsection{Proof of Lemma \ref{lem:mian-bounding-I2}}\label{sec:proof_G_upper_subg_fix_sigma}
We will need  the following property of Gaussian distributions.
\begin{lemma}[\cite{abramowitz1965handbook}]
\label{fact:z^2/2}
For a Gaussian distributed random variable $Z$ with mean $\mu$ and variance $\sigma^2$, for $z>0$,
\begin{equation}
\label{eq:main-perporty-Gaussian}
    \PP( Z >\mu+z\sigma)\leq \frac{1}{2}\exp\bigg({-\frac{z^2}{2}}\bigg) \qquad and \qquad  \PP( Z <\mu-z\sigma)\leq \frac{1}{2}\exp\bigg({-\frac{z^2}{2}}\bigg)
\end{equation}
\end{lemma}
\begin{proof}[Proof of Lemma \ref{lem:mian-bounding-I2}]
We decompose the proof of Lemma \ref{lem:mian-bounding-I2} into  the proof of the following two statements: (i) there exists a universal constant $c'$ such that 
\begin{equation}
\label{eq:<2/epe2}
\EE\bigg[ \frac{1}{G'_{1s}(\epsilon)}-1 \bigg] \leq {c'}, \quad\forall s,
\end{equation}
and (ii) for $L=\lceil32/\epsilon^2 \rceil$, it holds that 
\begin{equation}
\label{eq:>2/eps2}
 \EE \bigg[ \sum_{s=L}^{T} \bigg( \frac{1}{G'_{1s}(\epsilon)}-1 \bigg) \bigg]  \leq \frac{4}{e^2}\bigg(1+\frac{16}{\epsilon^2}\bigg).
\end{equation}
Let $\Theta_s=\mathcal{N}(\hat{\mu}_{1s},1/(\rho s))$ and $Y_s$ be the random variable  denoting the number of consecutive
independent trials until a sample of $\Theta_s$ becomes greater than $\mu_1-\epsilon$.
Note that $G'_{is}(\epsilon)=\PP(\theta\geq \mu_1-\epsilon)$, where $\theta$ is sampled from $\Theta_s$. Hence we have
\begin{equation}
   \EE\bigg[ \frac{1}{G'_{1s}(\epsilon)}-1 \bigg]= \EE[Y_s].
\end{equation}
Consider an integer $r\geq 1$. Let $z=\sqrt{2\rho' \log r}$, where $\rho'\in (\rho,1)$ and will be determined later. Let random variable $M_r$ be the maximum of $r$ independent samples from $\Theta_s$. Define $\cF_s$ to be the filtration consisting the history of plays of Algorithm \ref{alg:mots} up to the $s$-th pull of arm $1$. Then it holds 
\begin{align}\label{eq:Gu0003}
     \PP(Y_s< r)& \geq \PP(M_r > \mu_1-\epsilon)  \notag \\
        & \geq \EE\bigg[ \EE\bigg[ \bigg(M_r>\hat{\mu}_{1s}+\frac{z}{\sqrt{\rho s}},   \hat{\mu}_{1s}+\frac{z}{\sqrt{\rho s}}\geq \mu_1-\epsilon \bigg) \bigg| \cF_s \bigg] \bigg] \notag \\
                &=\EE \bigg[\ind \bigg\{\hat{\mu}_{1s}+\frac{z}{\sqrt{\rho s}} \geq \mu_1-\epsilon \bigg\} \cdot \PP\bigg( M_r > \hat{\mu}_{1s}+\frac{z}{\sqrt{\rho s}} \bigg| \cF_s\bigg) \bigg].
\end{align}
For a random variable $Z\sim\cN(\mu,\sigma^2)$,  from~\citet{abramowitz1965handbook}, we have
\begin{equation}
\label{eq:lowerbound-G}
  \frac{e^{-x^2/2}}{x\cdot \sqrt{2\pi}} \geq \PP(Z>\mu+x\sigma)\geq \frac{1}{\sqrt{2\pi}}\frac{x}{x^2+1} e^{-\frac{x^2}{2}}.
\end{equation}  
Therefore, if $r>e^2$, it holds that \begin{align}
    \PP \bigg( M_r>\hat{\mu}_{1s}+\frac{z}{\sqrt{\rho s}}\bigg| \cF_s \bigg)& \geq 1- \bigg(1-\frac{1}{\sqrt{2\pi}}\frac{z}{z^2+1} e^{-z^2/2} \bigg)^r \notag \\
    & = 1-\bigg(1-\frac{r^{-\rho'}}{\sqrt{2\pi}}\frac{\sqrt{2\rho' \log r}}{2\rho' \log r+1}  \bigg)^r \notag \\
    & \geq 1- \exp \bigg( -  \frac{r^{1-\rho'}}{\sqrt{8 \pi \log r}}    \bigg),
\end{align}
where the last inequality is due to $(1-x)^r\leq e^{-rx}$, $2\rho'\log r+1 \leq2\sqrt{2}\rho'\log r $ (since $r>e^2$ and $\rho'>1/2$) and $\rho'<1$.
Let $x=\log r$, then
\begin{align*}
   \exp \bigg( - \frac{r^{1-\rho'}}{\sqrt{8 \pi \log r}} \bigg) &\leq \frac{1}{r^2} 
   \qquad\Leftrightarrow\qquad\exp((1-\rho')x)\geq 2\sqrt{8\pi}x^{\frac{3}{2}}.
\end{align*}
It is easy to verify that for $x \geq 10/(1-\rho')^2$, $\exp((1-\rho')x)\geq 2\sqrt{8\pi}x^{\frac{3}{2}}$. Hence, if $r\geq \exp(10 /(1-\rho')^2)$, we have $\exp(-r^{1-\rho'}/({\sqrt{8 \pi \log r}})) \leq 1/r^2$.
\\
For $r\geq \exp(10 /(1-\rho')^2)$, we have
\begin{align}\label{eq:Gu0001}
    \PP \bigg( M_r>\hat{\mu}_{1s}+\frac{z}{\sqrt{\rho s}}\bigg| \cF_s\bigg) \geq 1-\frac{1}{r^2}.
\end{align}
For any $\epsilon>0$, it holds that
\begin{align}\label{eq:Gu0002}
       \PP\bigg(\hat{\mu}_{1s}+ \frac{z}{\sqrt{\rho s}}\geq \mu_1-\epsilon\bigg) &\geq  \PP\bigg(\hat{\mu}_{1s}+\frac{z}{\sqrt{\rho s}}\geq \mu_1\bigg)  \notag\\
       & \geq  1-\exp ( -z^2/(2\rho)) \notag\\
       &= 1-\exp (-\rho'/\rho \log r )\notag\\
       & =1- r^{-\rho'/\rho}.
\end{align}
where the second equality is due to Lemma~\ref{lem:subguassian}. Therefore, for $r\geq \exp[10 /(1-\rho')^2]$, substituting \eqref{eq:Gu0001} and \eqref{eq:Gu0002} into \eqref{eq:Gu0003} yields
\begin{align}
    \PP(Y_s<r)\geq 1-r^{-2}-r^{-\frac{\rho'}{\rho}}.
\end{align}
For any $\rho'>\rho$, this gives rise to
\begin{align*}
    \EE [Y_s ] & =\sum_{r=0}^{\infty}\PP(Y_s\geq r) \notag \\
    & \leq \exp\bigg[\frac{10}{(1-\rho')^2}\bigg]+\sum_{r\geq 1} \frac{1}{r^2}+\sum_{r\geq 1} r^{-\frac{\rho'}{\rho}} \notag \\
    & \leq \exp\bigg[\frac{10}{(1-\rho')^2}\bigg]+2+1+\int_{x=1}^{\infty}  x^{-\frac{\rho'}{\rho}} \dd x \notag \\
    & \leq  2\exp\bigg[\frac{10}{(1-\rho')^2}\bigg]  +\frac{1}{(1-\rho)-(1-\rho')}, 
    \end{align*}
 Let $1-\rho'=(1-\rho)/2$. We further obtain  
    \begin{align}\label{eq:using-delta}
        \EE \bigg[ \frac{1}{G'_{1s}(\epsilon)}-1 \bigg] \leq  2\exp\bigg[\frac{40}{(1-\rho)^2}\bigg]+\frac{2}{1-\rho}.
    \end{align}
Since $\rho \in (1/2,1)$ is fixed, then there exists a universal constant $c'>0$ such that
\begin{equation}
\label{eq:using-delta-res}
    \EE \bigg[ \frac{1}{G'_{1s}(\epsilon)}-1 \bigg] \leq c'.
\end{equation}
\nop{Note that $z=\sqrt{\log r}$. Hold on above equation, we further have
 For Gaussian random variable $Z$ with mean $\mu$ and variance $\sigma^2$, using Formula 7.1.13 from~\cite{abramowitz1948handbook}
\begin{equation}
    \PP(Z>\mu+x\sigma)\geq \frac{1}{\sqrt{2\pi}}\frac{x}{x^2+1} e^{-\frac{x^2}{2}}.
\end{equation}  
Therefore 
\begin{align}
    \PP \bigg( M_r>\hat{\mu}_{1s}+\frac{z}{\sqrt{s}}\bigg)& \geq 1- \bigg(1-\frac{1}{\sqrt{2\pi}}\frac{z}{z^2+1} e^{-z^2/2} \bigg)^r \notag \\
    & = 1-\bigg(1-\frac{1}{\sqrt{2\pi}}\frac{\sqrt{\log r}}{\log r+1} \frac{1}{\sqrt{r}} \bigg)^r \notag \\
    & \geq 1-e^{-\frac{r}{\sqrt{4\pi r \log r}}}
\end{align}}
Now, we turn to prove~\eqref{eq:>2/eps2}. Let $E_s$ be the event that $\hat{\mu}_{1s}\geq \mu_1-\epsilon/2$.  Let $X_{1s}$ is $\mathcal{N}(\hat{\mu}_{1s},1/(\rho s))$ distributed random variable. Using the upper bound of Lemma~\ref{fact:z^2/2} with $z=\epsilon/(2\sqrt{1/(\rho s)})$, we obtain 
\begin{align}
        \PP(X_{1s}>\mu_1-\epsilon \mid E_s)& \geq\PP( X_{1s}>\hat{\mu}_{1s}-\epsilon/2 \mid E_s) \geq 1-1/2\exp(-s\rho\epsilon^2/8). 
\end{align}
Then, we have
\begin{equation}\label{eq:lemma4_case2}
\begin{split}
    \EE \bigg[   \frac{1}{G'_{1s}(\epsilon)}-1 \bigg] & = \EE_{\hat{\mu}_{1s}\sim \Theta_s} \bigg[ \frac{1}{\PP(X_{1s}>\mu_1-\epsilon)}-1\bigg| \hat{\mu}_{1s}\bigg] \\
      & \leq \EE \bigg[ \frac{1}{\PP(X_{1s}>\mu_1-\epsilon\mid E_s)\cdot\PP (E_s)}-1\bigg] \\
      & \leq \EE \bigg[\frac{1}{(1-1/2\exp({-s\rho\epsilon^2/8}))\PP(E_s)}-1 \bigg].
    \end{split}
\end{equation}
Recall $L=\lceil 32/\epsilon^2 \rceil$. Applying Lemma~\ref{lem:subguassian}, we have
\begin{align}
\PP(E_s) &= 
    \PP\bigg(\hat{\mu}_{1s}\geq \mu_1-\frac{\epsilon}{2}\bigg)\geq 1- \exp \bigg(-\frac{s\epsilon^2}{8}\bigg)\geq 1- \exp (-s\rho\epsilon^2/8).
\end{align}
Substituting the above inequality into \eqref{eq:lemma4_case2} yields
\begin{align*}
     \EE \bigg[ \sum_{s=L}^{T} \bigg( \frac{1}{G'_{1s}(\epsilon)}-1 \bigg) \bigg]  &\leq \sum_{s=L}^{T} \bigg[ \frac{1}{(1-\exp({-s\rho\epsilon^2/8}))^2}-1\bigg] \\ 
     &\leq \sum_{s=L}^{T} 4\exp\bigg({-\frac{s\epsilon^2}{16}}\bigg)  \\
     & \leq 4\int_{L}^{\infty}\exp \bigg(-\frac{s\epsilon^2}{16}\bigg)\dd s + \frac{4}{e^2}\\
     &\leq \frac{4}{e^2} \bigg(1+ \frac{16}{\epsilon^2}\bigg).
\end{align*}
The second inequality follows since $1/(1-x)^2-1\leq 4x$, for $x\leq 1-\sqrt{2}/2$ and $\exp(-L\rho\epsilon^2/8)\leq 1/e^2$. We complete the proof of Lemma~\ref{lem:mian-bounding-I2} by combining~\eqref{eq:<2/epe2} and \eqref{eq:>2/eps2}.
\end{proof}

\subsection{Proof of Lemma~\ref{lem:decom=false}} \label{sec:decom=false}

\begin{proof}
Recall that $\rho=1$. Let $g(y)$ be the PDF of Gaussian distribution $\cN(\hat{\mu}_{1s},1/s)$.
\begin{align*}
      \EE\bigg[ \frac{1}{G'_{1s}(\epsilon)}-1 \bigg] &= \int_{-\infty}^{\infty} g(y) \bigg[\frac{1}{G'_{1s}(\epsilon)}-1 \bigg] \dd y \\
      & \geq \int_{-\infty}^{0} g(y) \bigg[\frac{1}{G'_{1s}(\epsilon)}-1 \bigg] \dd y \\
       & \geq -\sqrt{s} \sqrt{2\pi} \int_{-\infty}^{0} \bigg[ g(y){y}{}\exp\bigg(\frac{sy^2}{2} \bigg) \bigg] \dd y  \\
       & =   \int_{-\infty}^{0} \bigg[-sy e^{-\frac{s(y-\epsilon)^2}{2}}e^{\frac{sy^2}{2}}\bigg]  \dd y \\
       & = s e^{-\frac{s\epsilon^2}{2}} \int_{0}^{\infty} y e^{-s y\epsilon} \dd y \\
       & = s e^{-\frac{s\epsilon^2}{2}} \bigg( \frac{-ye^{-sy\epsilon}}{s\epsilon}+\frac{-e^{-sy\epsilon}}{s^2\epsilon^2}\bigg) \bigg|_{0}^{\infty} =\frac{e^{-{s\epsilon^2}/{2}}}{s\epsilon^2},
       \end{align*}
       where the second inequality is from~\eqref{eq:lowerbound-G}.
Let $\epsilon=\Delta_i/2=\sqrt{K\log T/T}$ for $i\in \{2,3,\cdots,K\}$. Since $K\log T\leq \sqrt{T}$,  we have
\begin{align}
\label{eq:decomp-last}
 K\Delta_i \cdot \EE\bigg[ \sum_{s=1}^{T-1}\bigg(\frac{1}{G'_{1s}(\epsilon)}-1\bigg)\bigg]& \geq K\Delta_i \cdot \sum_{s=1}^{\sqrt{T}} \frac{e^{-s\epsilon^2/2}}{s\epsilon^2} \\ \notag 
 & =\Omega \bigg(K\Delta_i \cdot \int_{s=1}^{\sqrt{T}} \frac{1}{s\epsilon^2} \dd s \bigg)= \Omega(\sqrt{K T\log T}),
\end{align}
which completes the proof.
\end{proof}

\subsection{Proof of Lemma~\ref{lem:asym1}}\label{sec:proof_asym1}

\begin{proof}
Since $\epsilon_T+\epsilon<\Delta_i$, we have $\mu_i+\epsilon_T\leq \mu_1-\epsilon$. Applying Lemma~\ref{lem:subguassian},  we have  $
\PP( \hat{\mu}_{is}>\mu_i+\epsilon_T) \leq \exp(-s\epsilon_T^2/2)$. Furthermore,
\begin{equation}
\label{eq:bounging-ui-eps}
\sum_{s=1}^{\infty}  \exp\bigg(-\frac{s\epsilon_T^2}{2}\bigg) \leq \frac{1}{\exp(\epsilon_T^2/2)-1}  \leq \frac{2}{\epsilon_T^2}.
\end{equation}
where the last inequality is due to the fact $1+x\leq e^x$ for all $x$. Define $L_{i}=2 \log T/(\rho(\Delta_i-\epsilon-\epsilon_T)^2)$.
 For $s\geq L_{i}$ and $X_{is}$ sampled from $\mathcal{N}(\hat{\mu}_{is},1/(\rho s))$, if $\hat{\mu}_{is}\leq\mu_i+\epsilon_T$, then  using Gaussian tail bound in~Lemma~\ref{fact:z^2/2}, we obtain
\begin{align}
\label{eq:bounding-ui-}
  \PP(X_{is}\geq \mu_1-\epsilon) & \leq \frac{1}{2} \exp \bigg(-\frac{\rho s(\hat{\mu}_{is}-\mu_1+\epsilon)^2}{2} \bigg)  \notag \\
    & \leq \frac{1}{2} \exp \bigg(-\frac{\rho s(\mu_1-\epsilon-\mu_i-\epsilon_T)^2}{2} \bigg) \notag \\
    &= \frac{1}{2} \exp \bigg(-\frac{\rho s(\Delta_i-\epsilon-\epsilon_T)^2}{2} \bigg) \notag \\
    & \leq \frac{1}{T}.
\end{align}
Let $Y_{is}$ be the event that $\hat{\mu}_{is}\leq\mu_i+\epsilon_T$ holds.
We further obtain
\begin{align}
\mathbb{E}\bigg[\sum_{s=1}^{T-1}\ind\{G'_{is}(\epsilon)>1/T\}\bigg] 
 \leq  & \mathbb{E}\bigg[\sum_{s=1}^{T-1}[\ind\{G'_{is}(\epsilon)>1/T\} \mid Y_{is}]\bigg] +\sum_{s=1}^{T-1}(1-\PP[Y_{is}]) \notag \\
\leq & \sum_{s=\lceil L_i \rceil} ^T\EE\bigg[[\ind \{\PP({X}_{is}>\mu_1-\epsilon)>1/T \}|Y_{is}] \bigg]+\lceil L_i \rceil+\sum_{s=1}^{T-1}(1-\PP[Y_{is}])\notag \\ 
\leq & \lceil L_i \rceil+\sum_{s=1}^{T-1}(1-\PP[Y_{is}]) 
    \leq   1+\frac{2}{\epsilon_T^2}+\frac{2\log T}{\rho(\Delta_i-\epsilon-\epsilon_T)^2}.
  \end{align} 
  where the first inequality is due to the factor $\PP(A)\leq \PP(A|B)+1-\PP(B)$, the third inequality is from~\eqref{eq:bounding-ui-} and the last inequality is from \eqref{eq:bounging-ui-eps}.
\end{proof}

\subsection{Proof of Lemma~\ref{lem:bounding-I_2-notfixed}}\begin{proof}
The proof of  Lemma~\ref{lem:bounding-I_2-notfixed} is the same as that of Lemma~\ref{lem:mian-bounding-I2}, except that the upper bound in~\eqref{eq:using-delta-res} will depend on $\rho$ instead of an absolute constant $c'$. In particular, plugging $\rho=1-\sqrt{40/\ilog^{(m)}(T)}$ back into \eqref{eq:using-delta} immediately yields
\begin{align}
        \EE \bigg[ \frac{1}{G'_{1s}(\epsilon)}-1 \bigg] & \leq  2\exp\bigg[\frac{40}{(1-\rho)^2}\bigg]+\frac{2}{1-\rho} \notag \\
        & \leq 2 \ilog^{(m-1)} (T)+2\ilog^{(m)} (T).
    \end{align}
    Therefore, there exists a constant $c'$ such that 
    \begin{align}\label{eq:<2/epe2_varying_delta}
        \EE \bigg[ \frac{1}{G'_{1s}(\epsilon)}-1 \bigg]\leq c'\ilog^{(m-1)} (T).
    \end{align}
    Thus, combining \eqref{eq:<2/epe2_varying_delta} and \eqref{eq:>2/eps2}, we obtain that
     \begin{align*}
      \sum_{s=1}^{T-1}\EE \bigg[ \frac{1}{G'_{1s}(\epsilon)}-1 \bigg]\leq O\bigg(\frac{\ilog^{(m-1)} (T)}{\epsilon^2}\bigg),
 \end{align*}
 which completes the proof.
\end{proof}

\subsection{Proof of Lemma~\ref{lemma:high_prob_all_positive1}}
We will need the following property of Gaussian distributions. 
\begin{lemma}[Lemma 12 of~\cite{lattimore2018refining}]
\label{lemma-refining}
Let $Z_1,Z_2,\cdots$ be an infinite sequence of independent standard Gaussian random variables and $S_n=\sum_{s=1}^n Z_s$. Let $d\in\{1,2,\cdots\}$ and $\Delta>0$, $\gamma>0$, $\lambda\in[0,\infty]^d$ and $h_{\lambda}(s)=\sum_{i=1}^d \min\{s,\sqrt{s\lambda_i} \}$, then 
\begin{align}
    \PP\bigg(\exists \ s\geq0: S_s\leq -\sqrt{2s\log^+\bigg(\frac{\gamma}{h_{\lambda}(s)} \bigg)} -t\Delta \bigg) \leq \frac{4h_{\lambda}(1/\Delta^2)}{\gamma}.
\end{align}
\end{lemma}

\begin{proof}[Proof of Lemma~\ref{lemma:high_prob_all_positive1}]
Using Lemma~\ref{lemma-refining} with {$\gamma=T/K$}, $d=1$ and $\lambda_1=\infty$, we have 
\begin{align}
    \PP\bigg( \exists s\geq 1: \hat{\beta}_s+\sqrt{\frac{2}{s}\log^+\bigg(\frac{T}{sK}\bigg)}+\Delta \leq 0\bigg)\leq \frac{4K}{T\Delta^2}.
\end{align}
 Note that for $\alpha\geq 2$ 
\begin{align}
    \sqrt{\frac{2}{s}\log^+\bigg(\frac{T}{sK} \bigg)}\leq \sqrt{\frac{\alpha}{s}\log^+\bigg(\frac{T}{sK} \bigg)},
\end{align}
Lemma follows.
\end{proof}

\subsection{Proof of Lemma \ref{lem:mian-bounding-I2-gaussian}}
Similar to the proof of Lemma \ref{sec:proof_G_upper_subg_fix_sigma}
, where we used the tail bound property of Gaussian distributions in
Lemma \ref{fact:z^2/2}, we need the following lemma for the tail bound of $\cJ$ distribution.
\begin{lemma}
\label{fact:Mz^2/2}
For a  random variable $Z\sim\cJ(\mu,\sigma^2)$, for any $z>0$,
\begin{equation}
\PP( Z>\mu+z\sigma)= \frac{1}{2}\exp\bigg({-\frac{z^2}{2}}\bigg) \qquad \text{and} \qquad \PP(Z<\mu-z\sigma)=\frac{1}{2}\exp\bigg({-\frac{z^2}{2}}\bigg).
\end{equation}
\end{lemma}
\begin{proof}[Proof of Lemma \ref{lem:mian-bounding-I2-gaussian}]
 Let $L=\lceil32/\epsilon^2 \rceil$. We decompose the proof of Lemma \ref{lem:mian-bounding-I2-gaussian} into  the proof of the following two statements: (i) there exists a universal constant $c'$ such that 
 
\begin{equation}
\sum_{s=1}^{L}\EE\bigg[ \frac{1}{G'_{1s}(\epsilon)}-1 \bigg] \leq \frac{c'}{\epsilon^2}, \quad\forall s,
\end{equation}
and (ii) it holds that 
\begin{equation}
 \EE \bigg[ \sum_{s=L}^{T} \bigg( \frac{1}{G'_{1s}(\epsilon)}-1 \bigg) \bigg]  \leq \frac{4}{e^2}\bigg(1+\frac{16}{\epsilon^2}\bigg).
\end{equation}
{Replacing Lemma~\ref{fact:z^2/2} by Lemma~\ref{fact:Mz^2/2}}, the rest of the proof for Statement (ii) is the same as that of \eqref{eq:>2/eps2} in the proof of Lemma~\ref{lem:mian-bounding-I2} presented in Section \eqref{sec:proof_G_upper_subg_fix_sigma}. Hence, we only  prove Statement (i) here.

Let $\hat{\mu}_{1s}=\mu_1+x$.
  Let $Z$ be a sample from $\cJ(\hat{\mu}_{1s},1/s)$. For $x<-\epsilon$, applying  Lemma~\ref{fact:Mz^2/2} with
$z=-\sqrt{s}(\epsilon+x)>0$ yields 
\begin{align}
\label{eq:jdistribution-Gis}
   {G'_{1s}(\epsilon)}=\PP(Z>\mu_1-\epsilon)=\frac{1}{2}\exp\bigg(-\frac{s(\epsilon+x)^2}{2} \bigg).
\end{align}

Since $\hat{\mu}_{1s}\sim \cN(\mu_1,1/s)$, $x\sim \cN(0,1/s)$. Let $f(x)$ be the PDF of $ \cN(0,1/s)$. Note that $G'_{1s}(\epsilon)$ is a random variable with respect to $\hat{\mu}_{1s}$ and $\hat{\mu}_{1s}=\mu_1+x$,  we have
\begin{align}
\label{eq:mian-J}
   \mathbb{E}_{x\sim\cN(0,1/s)}\bigg[\left(\frac{1}{G'_{1s}(\epsilon)}-1\right)\bigg] 
    =&\int_{-\infty}^{-\epsilon} f(x)\left(\frac{1}{G'_{1s}(\epsilon)}-1\right) \dd x + \int_{-\epsilon}^{\infty} f(x)\left(\frac{1}{G'_{1s}(\epsilon)}-1\right) \dd x \notag \\
    \leq& \int_{-\infty}^{-\epsilon} f(x)\left(2\exp(\frac{s(\epsilon+x)^2}{2})-1\right) \dd x \notag \\
       & + \int_{-\epsilon}^{\infty} f(x)\left(\frac{1}{G'_{1s}(\epsilon)}-1\right) \dd x \notag \\
     \leq& \int_{-\infty}^{-\epsilon} f(x)\left(2\exp(\frac{s(\epsilon+x)^2}{2})-1\right) \dd x +\int_{-\epsilon}^{\infty} f(x) \dd x  \notag \\
    \leq & \int_{-\infty}^{-\epsilon} \left(\sqrt{\frac{{2s}}{{\pi }}}\exp(\frac{-sx
    ^2}{2})\exp(\frac{s(\epsilon+x)^2}{2}) \right) \dd x+1 \notag \\
    \leq & \sqrt{\frac{{2s}}{{\pi }}}\exp\bigg(\frac{s\epsilon^2}{2}\bigg)\int_{-\infty}^{-\epsilon} \exp(s\epsilon x)\dd x +1 \notag\\
    \leq & \frac{e^{-s\epsilon^2/2}}{\sqrt{s}\epsilon}+1,
\end{align}
where the first inequality is due to~\eqref{eq:jdistribution-Gis}, the second inequality follows since $\hat{\mu}_{1s}=\mu_1+x\geq \mu_1-\epsilon$ and then   $G'_{1s}(\epsilon)=\PP(Z>\mu_1-\epsilon)\geq 1/2$.

Note that for $s\leq L$, $e^{-s\epsilon^2/2}= O(1)$. From~\eqref{eq:mian-J}, we immediately obtain that for $L=\lceil\frac{32}{\epsilon^2}\rceil$, we have 
\begin{align}
    \sum_{s=1}^{L} \mathbb{E}\bigg[\left(\frac{1}{G'_{1s}(\epsilon)}-1\right)\bigg] = O\bigg(\sum_{s=1}^{L} \frac{1}{\sqrt{s}\epsilon} \bigg)=O\bigg(\int_{s=1}^{1/\epsilon^2}  \frac{1}{\sqrt{s}\epsilon} \dd s\bigg)=O\bigg(\frac{1}{\epsilon^2}\bigg),
\end{align}
which completes the proof.
\end{proof}

\subsection{Proof of Lemma~\ref{lem:asym2}}
\begin{proof}
Replacing Lemma~\ref{fact:z^2/2} by Lemma~\ref{fact:Mz^2/2}, the rest of the proof for Lemma~\ref{lem:asym2} is the same as the proof of Lemma~\ref{lem:asym1} presented in Section \ref{sec:proof_asym1}. Thus we omit it for simplicity.
\end{proof}

\section{Tail Bounds for $\cJ$ Distribution}
In this section, we provide the proof of the tail bounds of $\cJ$ distribution.
\begin{proof}[Proof of Lemma \ref{fact:Mz^2/2}]
According to the PDF of $\cJ$ defined in \eqref{eq:def_pdf_of_J}, for any $z>0$, we immediately have
\begin{align}
       \PP(Z-\mu>z\sigma) & = \int_{z\sigma}^{\infty} \frac{1}{2\sigma^2}x\exp\bigg[{-\frac{1}{2}\bigg(\frac{x}{\sigma}\bigg)^2}\bigg]  \dd x \notag \\
       & =  \frac{-\sigma^2}{2\sigma^2}\exp\bigg[-\frac{x^2}{2\sigma^2} \bigg] \bigg|^{\infty}_{z\sigma} \notag  \\
       & = \frac{1}{2} \exp \bigg(-\frac{z^2}{2} \bigg).
\end{align}
Similarly, for any $z>0$, it holds that
\begin{equation}
\PP( Z<\mu-z\sigma)=\frac{1}{2}\exp\bigg({-\frac{z^2}{2}}\bigg),
\end{equation}
which completes the proof.
\end{proof}

\bibliographystyle{ims}
\bibliography{MOTS.bib}

\end{document}